\documentclass{article}

\usepackage[main, preprint]{neurips_2026}

\usepackage{url}
\usepackage{todonotes}
\usepackage{amsmath, amsfonts, amssymb, amsthm}
\usepackage{graphicx}
\usepackage{caption, subcaption}
\usepackage{booktabs}
\usepackage{multirow}
\usepackage{wrapfig}
\setcitestyle{authoryear,round,citesep={;},aysep={,},yysep={;}}
\usepackage[colorlinks=true,citecolor=blue,anchorcolor=purple]{hyperref}
\usepackage{amsmath}
\usepackage{amssymb}
\usepackage{verbatim}
\usepackage{tikz}
\usetikzlibrary{calc, positioning, arrows.meta, shapes.geometric}
\usepackage{cleveref}
\usepackage{multirow}
\usepackage[table]{xcolor}
\usepackage{enumitem}
\usepackage{adjustbox}
\usepackage{aliascnt}
\usepackage[T1]{fontenc}
\newcommand{\mymacro}[1]{#1}

\newcommand{\E}{\mymacro{\mathbb{E}}}
\newcommand{\Prob}{\mymacro{\mathbb{P}}}

\newcommand{\R}{\mymacro{\mathbb{R}}}

\newcommand{\N}{\mymacro{\mathbb{N}}}

\newcommand{\vw}{\mymacro{\mathbf{w}}}

\newcommand{\cM}{\mymacro{\mathcal{M}}}
\newcommand{\cW}{\mymacro{\mathcal{W}}}

\newcommand{\w}{\mymacro{\mathrm{w}}}
\newcommand{\x}{\mymacro{\mathrm{x}}}
\newcommand{\CTM}{\mymacro{\mathrm{CTM}}}
\newcommand{\BDM}{\mymacro{\mathrm{BDM}}}

\newcommand{\Ccal}{\mathcal{C}}

\newcommand{\q}{\mymacro{\mathbf{q}}}

\newcommand{\floor}[1]{\left\lfloor #1 \right\rfloor}

\newtheoremstyle{boldplain}
  {3pt}
  {3pt}
  {\itshape}
  {}
  {\bfseries}
  {.}
  {.5em}
  {\thmname{#1}\thmnumber{ #2}\thmnote{ \textbf{(#3)}}}

\newtheoremstyle{bolddefinition}
  {3pt}
  {3pt}
  {\normalfont}
  {}
  {\bfseries}
  {.}
  {.5em}
  {\thmname{#1}\thmnumber{ #2}\thmnote{ \textbf{(#3)}}}

\theoremstyle{boldplain}

\newaliascnt{lemma}{theorem}

\aliascntresetthe{lemma}

\newaliascnt{proposition}{theorem}
\newtheorem{proposition}[proposition]{Proposition}
\aliascntresetthe{proposition}

\newaliascnt{corollary}{theorem}
\newtheorem{corollary}[corollary]{Corollary}
\aliascntresetthe{corollary}

\theoremstyle{bolddefinition}

\newaliascnt{definition}{theorem}

\aliascntresetthe{definition}

\newaliascnt{assumption}{theorem}

\aliascntresetthe{assumption}

\newaliascnt{remark}{theorem}

\aliascntresetthe{remark}

\newaliascnt{example}{theorem}

\aliascntresetthe{example}

\title{Characterizing Learning in Deep Neural Networks using Tractable Algorithmic Complexity Analysis}

\author{%
  \textbf{Pedram Bakhtiarifard$^1$ \quad
          Sophia N. Wilson$^1$ \quad
          Mahmoud H. A. Afifi$^2$} \\
  \textbf{Jonathan Wenshøj$^1$ \quad
          Raghavendra Selvan$^1$} \\[0.8em]
  $^1$University of Copenhagen \\
  \texttt{\{pba,sophia.wilson,jowe,raghav\}@di.ku.dk} \\[0.6em]
  $^2$The American University in Cairo \\
  \texttt{mahmoudhossam@aucegypt.edu}
}

\begin{document}

\maketitle
\begin{abstract}
  Training large-scale deep neural networks (DNNs) is resource-intensive, making model compression a practical necessity. The widely accepted ``learning as compression'' hypothesis posits that training induces structure in network weights, which enables compression. Measuring this structure through Kolmogorov-Chaitin-Solomonoff (KCS) complexity is appealing, but existing estimators based on the Coding Theorem Method (CTM) and the Block Decomposition Method (BDM) are limited to small binary objects and do not scale to modern DNNs. We introduce the Quantized Block Decomposition method (QuBD), which extends algorithmic complexity estimation to any $k$-ary object. QuBD first quantizes the network weights to a finite alphabet, then estimates the KCS complexity by aggregating per bit-plane CTM estimates. We show theoretically that QuBD yields a tighter upper-bound on the estimation gap for ideal prefix KCS complexity under bit-plane independence than binarization-based methods. Using QuBD, we study how the algorithmic complexity of neural network weights evolves during training, showing that it decreases as models learn, scales with data budget, increases during overfitting, follows the delayed generalization observed during grokking, and correlates with generalization performance. We further show that algorithmic information resides predominantly in the most significant bit-planes, which can serve as a practical diagnostic for determining appropriate post-training quantization levels. This work offers novel insights into learning mechanisms in DNNs by providing the first scalable, tractable estimates of KCS complexity for large, non-binary objects such as DNN weights.\footnote{Code is available at \url{https://github.com/saintslab/QuBD}.}
\end{abstract}

\section{Introduction}
\label{sec:intro}

Training recent large-scale deep neural networks (DNNs) requires energy-intensive data centers that incur a large environmental footprint~\citep{strubell2019energy,anthony2020carbontracker,sevilla2022compute}. One way to address this is through the standard practice of compressing trained networks by removing the majority of their parameters (pruning), transferring their capabilities into smaller networks (knowledge distillation), or reducing the number of bits used to store each parameter (quantization)~\citep{bartoldson2023compute}. In most cases, such model compression techniques are able to drastically reduce memory, inference speed, and energy costs with only small drops in performance~\citep{gong2014compressing,hinton2014distilling, han2016deepcompressioncompressingdeep,nagel2021white}. Despite these benefits, compression is typically introduced as a separate stage after training, rather than as a property arising from the training dynamics themselves.

The modus operandi in DNN compression methods can be characterized as: \emph{start with bigger models to end with smaller ones}. One widely accepted hypothesis of ``learning as compression'' by~\cite{shwartz2017opening} seeks to explain the learning dynamics of DNNs that make such compression possible. 
Three lines of evidence support this view.
First, using the information bottleneck method,~\cite{tishby2000information} shows that a large part of training is devoted to learning more compact representations of the input, rather than improving performance directly. This suggests that the learned function depends on fewer distinct patterns, making it simpler to represent. 
Second, \cite{valledeep} finds that training biases the networks towards simpler functions.
Third, the lottery ticket hypothesis~\citep{frankle2019lotterytickethypothesisfinding} demonstrates that such simple functions can be realized by small sub-networks within the original model. However, all three accounts are mostly understood in the context of feedforward networks, and describe \emph{what} structure emerges (simpler representations and sparse sub-networks), rather than \emph{when} it emerges or \emph{how much} compression occurs.
\looseness=-1

\begin{figure}[t]
    \centering
    \includegraphics[width=\textwidth]{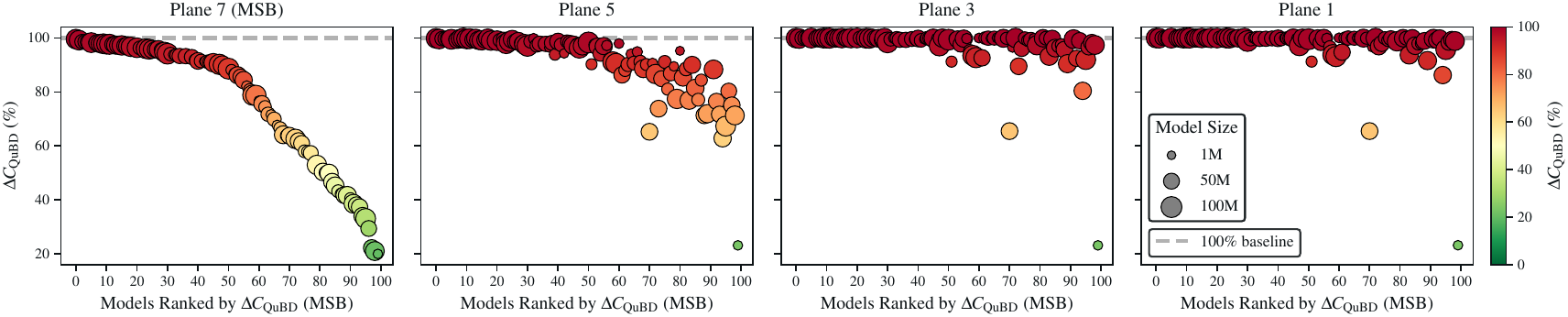}    
    \caption{Relative algorithmic complexity expressed as the ratio between the complexities of pretrained and random weights, $\Delta\Ccal_\text{QuBD}$ (\%), estimated using QuBD. We show the trend for 100 pretrained {\tt timm} models~\citep{rw2019timm} with up to 100M parameters using 8 bit-planes. Models are ordered by decreasing $\Delta\Ccal_\text{QuBD}$ for the most significant bit-plane (MSB). Further details can be found in Sec.~\ref{sec:learning}, Appx.~\ref{app:timm_bit_planes}, and Appx.~\ref{app:qubd_ratio}.}
    \label{fig:complexity}
    \vspace{-0.5cm}
\end{figure}

Recent work by~\cite{sakabe2025evaluating} attempts to address these gaps by characterizing learning in binary neural networks through the lens of algorithmic information theory~\citep{chaitin1977algorithmic}. They show that algorithmic complexity, estimated as structure in weights or the lack thereof, correlates strongly with learning dynamics. 
This provides direct empirical support for the ``learning as compression'' hypothesis: training induces structure in the model parameters, and this structure is precisely what enables compression. Building on this, \citet{bakhtiarifard2026momos} shows that such a structure can be actively enforced throughout training without performance loss, yielding model classes that are inherently more compressible.
\looseness=-1

Despite this promise, measuring algorithmic~\citep{kolmogorov1965three,chaitin1966length,solomonoff1960preliminary} (KCS) complexity is non-computable by definition and is only upper-semicomputable. Current KCS complexity estimators rely on the Coding Theorem Method (CTM), which estimates the output frequency of small Turing machines through large-scale simulations~\citep{soler2014calculating}. These simulations are limited to a small number of states and binary symbols, making the method practical only for small binary objects. The Block Decomposition Method (BDM) proposed by~\cite{zenil2018decomposition} scales this approach by estimating complexity from small sub-structures within larger objects. However, as we show through both theoretical and empirical analyses (Sec.~\ref{sec:method}), BDM does not scale to DNNs, thereby preventing a comparison of learning dynamics across model sizes and architectures.
\looseness=-1

To address this, we introduce the Quantized Block Decomposition method (QuBD), which quantizes any non-binary object into a finite set of integer values and estimates its KCS complexity through a bit-plane decomposition. This extends CTM-based estimators for floating-point objects like the weights of DNNs. We show theoretically that QuBD yields a tighter upper bound for estimating the ideal prefix KCS complexity of complex objects under bit-plane independence. Using QuBD, we provide the first scalable empirical characterization of the ``learning as compression'' hypothesis for DNNs beyond simple feedforward networks. We show empirically that KCS complexity decreases as the models are trained. This is visualized in Fig.~\ref{fig:complexity} for 100 models with up to 100M trainable parameters.  Furthermore, we observe that algorithmic complexity is lower in the most significant bit-planes, that it correlates with generalization, and that it can serve as a practical diagnostic for choosing model compression levels, such as precision for post-training quantization.

Our main contributions are:
\begin{itemize}[itemsep=0.5pt, topsep=3pt,leftmargin=*]
    \item We introduce QuBD, which combines quantization with a CTM-based estimator over bit-plane decomposition, making KCS complexity analysis feasible for large, non-binary objects.
    \item We prove that QuBD yields a smaller estimation gap relative to ideal prefix KCS complexity under bit-plane independence compared to binarization-based methods. 

    \item We use QuBD to study DNN learning dynamics of up to 1B parameters, showing that KCS complexity decreases as models learn and generalize, increases during overfitting, and tracks phenomena like grokking.

    \item We demonstrate that QuBD complexity serves as a practical diagnostic for model compression, as algorithmic information concentrates in the most significant bit-planes while those with no complexity reduction during training can be safely discarded. 
\end{itemize}

\section{Background and Related Work}
\label{sec:related}

\subsection{Algorithmic Complexity of Neural Networks} 
Different notions of model compression have been used to explain generalization. For instance, classical learning theory~\citep{vc99,bartlett2002rademacher} suggests that models with more parameters than data points should overfit, yet such models often generalize well in practice~\citep{rethink_zhang,neyshabur2017exploringgeneralizationdeeplearning,Belkin_2019}. Moreover, a large fraction of model parameters can be removed~\citep{hinton2014distilling}, shared~\citep{ullrich2017soft}, or quantized~\citep{han2016deepcompressioncompressingdeep} post-training with only small drops in performance.

Since neural networks admit many function-preserving reparametrizations~\citep{dinh2017sharpminimageneralizedeep}, quantities such as curvature or flatness can vary without changing the learned function, limiting how directly they explain generalization behavior~\citep{wilson2025deep,goldblum2024position}. This motivates the KCS complexity view of learning, where complexity and compressibility are studied through the description length of the learned object: the lower the KCS complexity, the shorter the generating description an object admits~\citep{li2008introduction,lotfi2022pac}. From this perspective,~\citet{valledeep} show that models exhibit a simplicity bias, favoring functions with low KCS complexity, explaining why structured targets are often learned in practice and tend to exhibit redundancy. Moreover, \citet{martin2025setolsemiempiricaltheorydeep} find that parameters of a trained model contain measurable correlation structure that evolves during learning and tracks generalization.

Although KCS complexity is non-computable due to the halting problem~\citep{turing1936computable}, approximations based on the algorithmic probability of an object~\citep{solomonoff1964formal}, computed from large-scale simulations of Turing machines~\citep{soler2014calculating, zenil2018decomposition}, have enabled recent works to assess this perspective for neural networks empirically. Such estimators track structural properties better compared to purely statistical measures like entropy~\citep{zenil2020algorithmic}, commonly used to explain properties of neural networks~\citep{xu2017information, achille2018emergenceinvariancedisentanglementdeep}. In particular,~\cite{sakabe2025evaluating} show that estimated KCS complexity tracks learning dynamics in binarized neural networks more closely than entropy-based measures. That complexity decreases throughout training, capturing changes in the algorithmic structure of the learned function and how such structure shapes the optimization landscape~\citep{bakhtiarifard2026momos}. However, existing estimators do not scale to modern networks, limiting their applicability and motivating the development of tractable KCS complexity estimators for DNNs.

\subsection{Algorithmic Information Theory}
\label{sec:ait}
The central idea formalized in KCS complexity is that structured objects admit shorter descriptions than truly random objects. For a
finite binary object $\vw\in\{0,1\}^*$, program $\mathrm{p}$, this is expressed as
\begin{equation}\label{eq:kc}
    \Ccal(\vw)
    :=
    \min\{|\mathrm p|:U(\mathrm p)=\vw\},
\end{equation}
where $U$ is a fixed universal prefix-free Turing machine and $|\mathrm p|$ denotes the length of the program in bits. Since determining whether all shorter programs produce $\vw$ and halt is undecidable~\citep{turing1936computable,davis1982computability}, KCS complexity is not computable in general.  

Instead, estimates are obtained via algorithmic probability through large-scale simulations of finite Turing machines, using the Coding Theorem~\citep{levin1974laws} as operationalized by~\cite{zenil2015twodimensionalkolmogorovcomplexityvalidation}. The CTM estimates the algorithmic probability of an object by its empirical output frequency across all simulations, and the collection of such estimations is referred to as a CTM table.

BDM~\citep{zenil2018decomposition} extends CTM to larger objects by partitioning objects into smaller blocks, estimating the complexity of each block using pre-computed CTM tables, and accounting for repeated blocks with a logarithmic multiplicity term.

Concretely, let $\vw$ be a large object partitioned by a lossless function $\Pi(\vw)=(\w_1,\ldots,\w_m)$ and let $\Pi_u(\vw)$ denote the set of distinct blocks of $\vw$. The KCS complexity is then estimated as
\begin{equation}\label{eq:kc_bdm}
    \Ccal_\BDM(\vw)
    =
    \sum_{\w_j\in\Pi_u(\vw)}
    \Ccal_\CTM(\w_j)+\log_2 n_{\w_j},
\end{equation}
where $n_{\w_j}$ is the number of occurrences of $\w_j \in \Pi(\vw)$ and $\Ccal_{\CTM}(\w_j)$ is the estimated KCS complexity of $\w_j$, derived from its empirical output frequency across Turing machine simulations.

Since CTM tables are defined over binary objects, non-binary objects must first be represented as such prior to partitioning and estimation, which is typically achieved through binarization. Furthermore, the block sizes supported by CTM tables are limited, e.g., $(4 \times 4)$ blocks, which creates a problem when meaningful structures are only revealed at larger scales.

Because scaling CTM tables to larger block sizes is computationally prohibitive and the supported two-dimensional blocks are only defined over binary objects, two fundamental limitations arise. First, the binarized representations may be lossy, discarding useful information about the original object, such that the complexity estimate may correspond to a different target object than intended. Second, the representation and partitioning determine which structures are visible in CTM tables; if meaningful structure is not exposed, it becomes difficult to distinguish genuinely complex objects from objects whose structure is simply obfuscated.
\looseness=-1

We provide further background on algorithmic complexity in Appx.~\ref {app:background}. All CTM estimates in this work are from the CTM tables in \texttt{pyBDM}~\citep{pybdmSoftware}, with a block size $\pi \in \{(2 \times 2), (3 \times 3), (4 \times 4) \}$. We use the largest $\pi = (4 \times 4)$, unless otherwise stated, as it is the most representative for the objects in consideration.

\section{Tractable Algorithmic Complexity Analysis using QuBD}
\label{sec:method}

\begin{wrapfigure}[18]{r}{0.475\textwidth}
\vspace{-1.6em}
\centering
\includegraphics[width=0.385\textwidth]{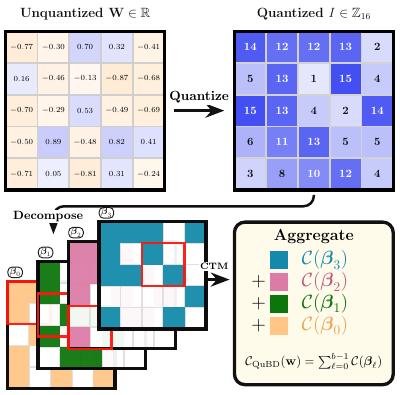}
\vspace{-0.1em}
\caption{QuBD first quantizes weights into finite codes \textbf{(top)}, then decomposes them into aligned bit-planes to aggregate the per-plane KCS estimates \textbf{(bottom)}.}
\label{fig:qubd}
\end{wrapfigure}

Our proposed KCS complexity estimation method, QuBD, shown in Fig.~\ref{fig:qubd}, first quantizes the model weights to a finite alphabet, then represents each quantized weight by its binary expansion across bit-planes. The KCS complexity estimation is then performed by aggregating the CTM-based estimations over all the bit-planes. We show theoretically that under bit-plane independence, this yields a tighter upper-bound estimate of ideal prefix KCS complexity than BDM-based methods that rely on binarization. 
We verify this empirically and further characterize the conditions under which such estimations remain informative across object scales, given the finite CTM tables.

Let $\vw = (\w_1, \ldots, \w_d) \in \R ^d$ denote a $d$-dimensional real object of model weights. We map $\vw$ to a finite alphabet $\Sigma_b$ using an affine uniform quantizer $Q_b : \R^d \to \Sigma_b^d$, described in Appx.~\ref{app:quantizer}. This yields the discrete object
\begin{equation}
\mathbf{q}_b = Q_b(\vw) = \bigl(q_{b,1},\ldots, q_{b,d}\bigr) \in \Sigma_b^d, \text{ with } \Sigma_b := \{0,\dots,2^b-1\},
\label{eq:object_quantized}
\end{equation}
where each weight $q_i$, with $i$ being one of $d$ dimensions, is a $b$-bit symbol of $2^b$ possible elements. We write $b^\star$ for the target lossless precision, and $\vw^{\star}$ for the corresponding target object. For instance, $b^\star=32$ for full-precision weights, while the representation is lossless when $b = b^\star$, though precision choices $b \leq 8$ are increasingly relevant for compute-efficient models~\citep{gholami2021surveyquantizationmethodsefficient}. 
Since $\mathbf{q}_b$ must be decomposed into smaller blocks supported by CTM tables, the choice of quantized representation determines which structures are exposed and therefore affects the theoretical gap between the KCS complexity of the target object $\vw^\star$ and the estimated object $\mathbf{q}_b$.

\paragraph{Lossless Exposure of Objects.}\label{paragraph:exposure}
One possible lossless representation of Eq.~\eqref{eq:object_quantized} is the serialized bit string $\mathbf{q}_b^{_\text{ser}} \in \{0,1\}^{d\cdot b}$, which for block size $\pi$ exposes $m = d\cdot b / \pi$ blocks. The KCS complexity of the original non-binary object and the CTM-estimate of the corresponding serialized bit-stream may be different due to the limitations of finite CTM tables (see Sec.~\ref{sec:ait}).
In contrast, binarization $\mathbf{q}_b^{_\text{bin}} \in \{0,1 \}^d$ exposes $b$ times fewer blocks but yields an unreliable estimate that differs significantly from that of the target object\footnote{When binarization is applied through the sign function, it defines a different target object. The additive-constant invariance of KCS complexity applies only to lossless changes of description of the same object~\citep{kolmogorov1965three}; a sign map is many-to-one and therefore falls outside that setting. However, we include it in this work for completeness.}.

The bit-plane representation uses the binary expansion of each quantized weight:
\begin{equation}\label{eq:bitplane1}
q_i = \sum_{\ell=0}^{b-1} 2^\ell \beta_{i,\ell},
\qquad
\beta_{i,\ell} \in \{0,1\},
\qquad
\boldsymbol{\beta_\ell} := (\beta_{1, \ell}, \ldots,\beta_{d,\ell}) \in \{0,1\}^d
\end{equation}
where $\beta_{i,b-1}$ denotes the Most Significant Bit (MSB) and $\beta_{i,0}$ denotes the Least Significant Bit (LSB) of $q_i$. The vector $\boldsymbol{\beta}_\ell$ collects the $\ell$-th bit of each of the $d$ symbols, forming the $\ell$-th bit-plane. The target object can thus be exposed one bit-plane at a time. The representation remains lossless when $b = b^\star$ since the bit-planes jointly and uniquely determine $\mathbf{q}_b$. The structural information captured by bit-planes has furthermore been shown to be useful for downstream tasks~\citep{nhan2025improvement}. We then define the Quantized Block Decomposition (QuBD) complexity estimate by
\begin{equation}
    \Ccal_{\text{QuBD}}(\vw) = \sum_{\ell=0}^{b-1} \left( \sum_{\w_j \in \prod(\boldsymbol{\beta}_\ell)} \Ccal_{\text{CTM}}(\w_j) + \log_2 n_{\w_j} \right) = \sum_{\ell=0}^{b-1} \Ccal ( \boldsymbol{\beta_\ell} ) .
\end{equation}

Interestingly, bit-planes can also be used to reveal structural information progressively, which is a common practice in digital image processing~\citep{chang1999new}. Consider discrete objects $q_i \in \mathbf{q}_b$, then for $1 \le k \le b$ define the $k$-bit MSB prefix index of each $q_i$ as
\begin{equation}\label{eq:bitplane_prefix}
c_i^{(k)} := \left\lfloor \frac{q_i}{2^{b-k}} \right\rfloor = \left\lfloor \frac{c_i^{(k+1)}}{2} \right\rfloor ,~~~\text{and}~~~ q_i \in \{c2^{b-k},\dots,(c+1)2^{b-k}-1 \}
\end{equation}
where $c_i^{(k)}=c$. 
For instance, if $b^\star=4$, $q_i=13=1101_2$,  then retaining one bit gives the interval $\{8,\dots,15\}$, retaining two bits gives $\{12,\dots,15\}$, and retaining three bits gives $\{12,13\}$.

\subsection{Quantized Block Decomposition (QuBD) Complexity}\label{sec:qubd_theory}
A fixed lossless representation changes the KCS complexity of an object only by an additive constant~\citep{kolmogorov1965three}. In contrast, retaining only the first $b$ bit-planes for $1 \le b < b^\star$ is not lossless. Crucially, even when $b=b^\star$, the object cannot be estimated directly and must be estimated from its decomposition, so its KCS complexity depends on the chosen representation. We consider the change in the target object induced by quantization and ask: \emph{What is the loss in KCS complexity between the target object and the estimated quantized object?}

To formalize this estimation gap, let $\vw^\star$ be the target object and consider its decomposition: 
\begin{equation}
\mathbf{q}_b := \floor{\frac{\vw^\star}{2^{(b^\star-b)}}} = (\mathbf{q}_{b,1},\ldots,\mathbf{q}_{b,d}) ,
\qquad
\mathbf{q}_{b,i} = \floor{\frac{\vw_i^\star}{2^{(b^\star-b)}}}
= \sum_{\ell=b^\star-b}^{b^\star-1} 2^{\ell-(b^\star-b)} \beta_{i,\ell}^\star .
\end{equation}
Here $\mathbf{q}_b$ retains exactly the $b$ MSB-planes of $\vw^\star$, while discarding the lower $b^\star - b$ bit-planes. For precision $b < b^\star$, define the omitted lower-bit residuals $\delta_b = (\delta_{b,1},\ldots,\delta_{b,d})$ by
\begin{equation}
\delta_b := \vw^\star - 2^{(b^\star-b)} \mathbf{q}_b .
\end{equation}
Here, $\delta_{b,i} = \vw_i^\star - 2^{(b^\star-b)} \mathbf{q}_{b,i}$ and $0 \le \delta_{b,i} < 2^{b^\star-b}$ for each $i$. Since $\mathbf{q}_b$ and $\delta_b$ capture the retained and omitted $b$ bit-planes, respectively, the target object can be estimated as the composite $(\mathbf{q}_b,\delta_b)$.

We analyze the theoretical gap between $\vw^\star$ and $(\mathbf{q}_b,\delta_b)$ using standard information-symmetry identities of KCS complexity from~\citet{li2008introduction}, summarized in Appx.~\ref{app:kc_facts}. An example of bit-plane decomposition is provided in Appx.~\ref{appendix:bitplane_example}.

\begin{proposition}[Residual Decomposition]\label{thm:res_decomp}
For a $d$-dimensional target object $\vw^\star \in \R^d$ with precision $1 \le b \le b^\star$, the ideal prefix KCS complexity is
\begin{equation}\label{thm:res_decomp_eq1}
\Ccal(\vw^\star) = \Ccal(\mathbf{q}_b) + R_b + O(\log (b^\star d)),
\qquad
R_b := \Ccal(\delta_b \mid \mathbf{q}_b),
\end{equation}
where $0 \le R_b \le (b^\star-b)d + O(1)$. Estimating $\vw^\star$ from $\mathbf{q}_b$ changes the complexity due to the residual information $R_b$ up to a logarithmic term. If $b=b^\star$, then $\delta_b=0$, $R_b=O(1)$, and the decomposition offers a target object differing only by a constant complexity. Proof in Appx.~\ref{proof:res_decomp}.
\end{proposition}

The inequalities in Prop.~\ref{thm:res_decomp} are necessarily tight since $R_b \to (b^\star -b)d + O(1)$ when there is information about the discarded bit-planes in the retained ones. On the other hand, $R_b = O(1)$ when there is no such information. So estimating $(\mathbf{q}_b, \delta_b)$ instead results in only a logarithmic error in bit length and in the residual information from the omitted bits. The following analysis demonstrates that the residual diminishes as additional bit-planes are incorporated, improving the upper-bound on the estimation gap when assuming independence across bit-planes.

\begin{proposition}[Bit-Plane Residual Loss]
\label{thm:add_bitplane}
For $1 \le b < b^\star$ and independent bit-planes $\boldsymbol{\beta_\ell} $, let
\begin{equation}
\boldsymbol{P}_{b+1} := (\beta^\star_{1,(b^\star-b)-1},\ldots,\beta^\star_{d,(b^\star-b)-1}) \in \{0,1\}^d ,
\end{equation}
where $\boldsymbol{P}_{b+1}$ is the next MSB plane after the $b$ retained in
$\mathbf{q}_b$. Then $\mathbf{q}_{b+1} = 2\mathbf{q}_b + \boldsymbol{P}_{b+1}$, and
\begin{equation}
\begin{aligned}
R_b &\ge R_{b+1} - O(\log(b^\star d)) .
\end{aligned}
\end{equation}
Thus, adding the $(b+1)$-th bit-plane cannot increase the residual loss, and more importantly, the residual loss decreases whenever bit-planes are independent and not identical. Proof in Appx.~\ref{proof:add_bitplane}.
\end{proposition}
Repeated application of Prop.~\ref{thm:add_bitplane} shows that if $b=b^\star$, only a constant error remains; we refer to Cor.~\ref{cor:cumulative_refinement} and its proof in Appx.~\ref{proof:cumulative_refinement}. Importantly, when $b>1$, the bit-plane decomposition yields a smaller residual than one-bit binarization under the independence assumption. And the impact of the quantization level on the decomposition does not scale beyond $O(\log(b^\star d))$. For clarity, note that the decomposition overhead remains because the refinement does not guarantee the same ordering.

We validate our findings in Appx.~\ref{app:residual_validation} (Fig.~\ref{fig:qubd_bitplane_loss}) by comparing the estimation gap to the reference CTM-table complexity across different representations. This shows that this gap narrows as QuBD retains more bit-planes. For clarity, validating directly on other objects, for instance, model weights, would confound estimator saturation with the structure induced by training noted by~\cite{valledeep}.
We next analyze how QuBD complexity scales with CTM-tables, providing more informative estimates at larger scales.

\subsection{Saturation in Finite-state Machines}\label{sec:saturation_effect}
\label{sec:saturation}
Let the binary block size $\pi$ be fixed, and define $S_\pi\subseteq \{0,1\}^{\pi}$ as the finite CTM support. For an exposure $z_{b,r}=\phi(\mathbf{q}_b)$, the number of
blocks $m$ is determined after the quantized object is exposed by representation $r$ and partitioned into binary blocks, $\Pi_\pi(z_{b,r})=(\w_1,\ldots,\w_m)$, each of size $\pi$. The same $d$ quantized symbols may yield a different number of binary blocks depending on $r$. Define $S_{\pi,r}^{(b)}\subseteq S_\pi$ as the set of supported blocks reachable under this exposure, and let $a_\pi(z_{b,r})$ denote the number of reachable blocks that appear among $(\w_1,\ldots, \w_m)$. We then define
\begin{equation}
H_{\pi,r}^{(b)}:=|S_{\pi,r}^{(b)}|,
\qquad
\sigma_{\pi,r}(m):=
\frac{\E[a_\pi(z_{b,r})]}{H_{\pi,r}^{(b)}} .
\end{equation}
Here, $\sigma_{\pi,r}(m)$ represents the expected fraction of reachable support that appears after $m$ exposed blocks. The estimate is considered $(1-\epsilon)$-saturated for some $\epsilon \in [0,1]$ when $\sigma_{\pi,r}(m)\ge 1-\epsilon$. Additional details are provided in Appx.~\ref{app:saturation_details}.

\begin{proposition}[Finite-table Saturation] \label{prop:finite_saturation}
Assume the $m$ exposed blocks are sampled independently and uniformly from
$S_{\pi,r}^{(b)}$. Then
\begin{equation}
\sigma_{\pi,r}(m)
=
1-\left(1-\frac{1}{H_{\pi,r}^{(b)}}\right)^m
\approx
1-e^{-m/H_{\pi,r}^{(b)}} ,
\qquad
m_\epsilon \gtrsim H_{\pi,r}^{(b)}\log\frac{1}{\epsilon},
\end{equation}
where $m_\epsilon$ is the number of exposed blocks expected for
$(1-\epsilon)$-saturation. Proof in Appx.~\ref{proof:finite_saturation}.
\end{proposition}

By Prop.~\ref{prop:finite_saturation}, exposure $r$ saturates when
$m_r\approx H_{\pi,r}^{(b)}\log(1/\epsilon)$. For $d$ quantized symbols,
$m_{\text{bin}}=d/\pi$, $m_{\mathrm{ser}}=d\cdot b/\pi$, and
$m_{\text{plane}}=d/\pi$ per bit-plane, which yields
\begin{equation}
d_{\mathrm{bin}}\gtrsim \pi H_{\pi,\mathrm{bin}}\log\frac{1}{\epsilon},
\qquad
d_{\mathrm{ser}}\gtrsim \frac{\pi H_{\pi,\mathrm{ser}}^{(b)}}{b}\log\frac{1}{\epsilon},
\qquad
d_{\mathrm{plane}}\gtrsim \pi H_{\pi,\mathrm{plane}}\log\frac{1}{\epsilon}.
\end{equation}
While binarization saturates more slowly, it incurs a larger residual loss. In contrast, serialization is lossless but saturates earlier, as it exposes $b$ times more blocks to a single finite CTM table. Bit-plane exposure estimates the same target object by scaling the support according to the number of bit-planes while evaluating local structures independently across its bit-planes. The corresponding saturation rates are illustrated in Appx.~\ref{app:saturation_exposure} (Fig.~\ref{fig:saturation_limits}).

\paragraph{When Representations Fail.}
The problem with the saturation effect (Sec.~\ref {sec:saturation}) is that $a_\pi(z_{b,r})$ stops growing while $m$ keeps increasing, to the point where only multiplicity is counted (recall Eq.~\eqref{eq:kc_bdm}). While this is reasonable for simple objects, it is not for structurally rich or large objects, where additional local structure can no longer be captured. Hence, the estimate becomes uninformative, similar to the entropy collapse described in~\cite{sakabe2025evaluating}.

Representation failure is illustrated in Fig.~\ref{fig:structure_estimate}, which compares a structured $8$-ary object $\mathbf{x}=(\x_1,\ldots,\x_d)$ with its partially permuted counterpart $\mathbf{x}_\rho$ using $\pi = (4\times 4)$, where $\rho$ is the fraction of positions permuted.  Let $n_a:=|\{i:\x_i=a\}|$ be the number of occurrences of symbol $a\in\Sigma_b$ in $\mathbf{x}$.
\begin{wrapfigure}{r}{0.6\textwidth}
    \centering
    \includegraphics[width=\linewidth]{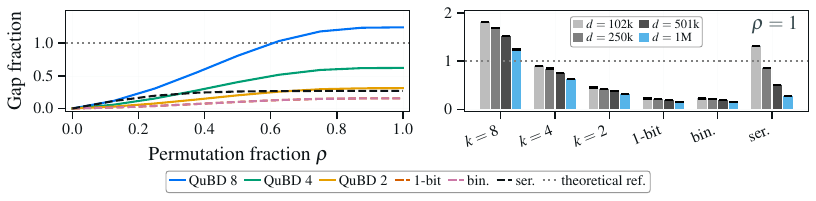}
    \caption{Complexity gap between estimations of an ordered $8$-ary object and its permutation, where $\rho$ is the fraction of symbols $d$ permuted. The gap is shown \textbf{(left)} for an object with $d = 1M$ symbols as $\rho$ increases, and \textbf{(right)} across object sizes for different methods when $\rho =1$. }
    \label{fig:structure_estimate}
    \vspace{-0.25cm}
\end{wrapfigure}
Since the permutation preserves all symbol counts and alters only their positions, any difference in complexity estimates must arise from the change in arrangement. The information required to identify one such arrangement is $\log_2(d!/\prod_a n_a!)$ bits
\citep{cover1999elements,li2008introduction}, which serves as a reference for the complexity added by the permutation. The values in Fig.~\ref{fig:structure_estimate} are normalized by the object type class described in~\cite{cover1999elements}. An informative estimator should assign higher complexity to $\mathbf{x}_\rho$ than to $\mathbf{x}$, with the gap increasing in $\rho$. If the KCS complexity gap $\Ccal(\mathbf{x}_\rho)-\Ccal(\mathbf{x})$ collapses toward zero, the estimator fails to identify structure. As shown in Fig.~\ref{fig:structure_estimate}, QuBD maintains a larger and more stable estimation gap than serialization, one-bit, and sign-binarization. Similar results are found for CTM tables of $\pi = (3 \times 3)$ and $\pi = (3 \times 3)$, illustrated in Appx.~\ref{app:saturation_details} (Fig.~\ref{fig:rep_failure}) Importantly, in the saturated regime, retaining even a small fraction of the reference gap still yields a distinguishable measure of lower- and higher-complexity models. This regime is shown in Appx.~\ref{app:saturation_exposure} (Fig.~\ref{fig:rep_failure_100M}), and additional results in Appx.~\ref{app:pythia}, finding similar improvements in KCS estimation with QuBD.

\section{Characterizing Learning in Deep Neural Networks}
\label{sec:learning}

We normalize the QuBD complexity estimate of an object $\vw$ relative to a random object of the same length $\tilde{\vw}$, denoted $\Delta\Ccal_\text{QuBD}$, for easier interpretation:
\begin{equation}
\Delta\Ccal_\text{QuBD} = \frac{\Ccal_{\text{QuBD}}(\vw)}{\Ccal_{\text{QuBD}}(\tilde{\vw})}.
\label{eq:delta_c}
\end{equation}
A truly random object has $\Delta\Ccal_\text{QuBD}\approx 1$, whereas an algorithmically simple object\footnote{Strictly speaking, $\Delta\Ccal_\text{QuBD}\approx 0$ holds only for sufficiently large objects, as even maximally simple objects such as an all-zero string retain some complexity. For objects of sizes ($8 \times 8$), ($16 \times 16$), ($32 \times32$), $\Delta\Ccal_\text{QuBD}$ for an all-zero string is 0.19, 0.05, and 0.01, respectively.} has $\Delta\Ccal_\text{QuBD}\approx 0$.

\paragraph{Learning Reduces Algorithmic Complexity of Networks.} Neural network weights are initialized randomly and, from an algorithmic complexity perspective, constitute highly complex objects because they lack structure and therefore cannot be compressed. This raises a natural question: {\em What happens to the algorithmic complexity of the weights as networks train?}

\begin{figure}[htbp]
    \centering
    \includegraphics[width=0.325\linewidth]{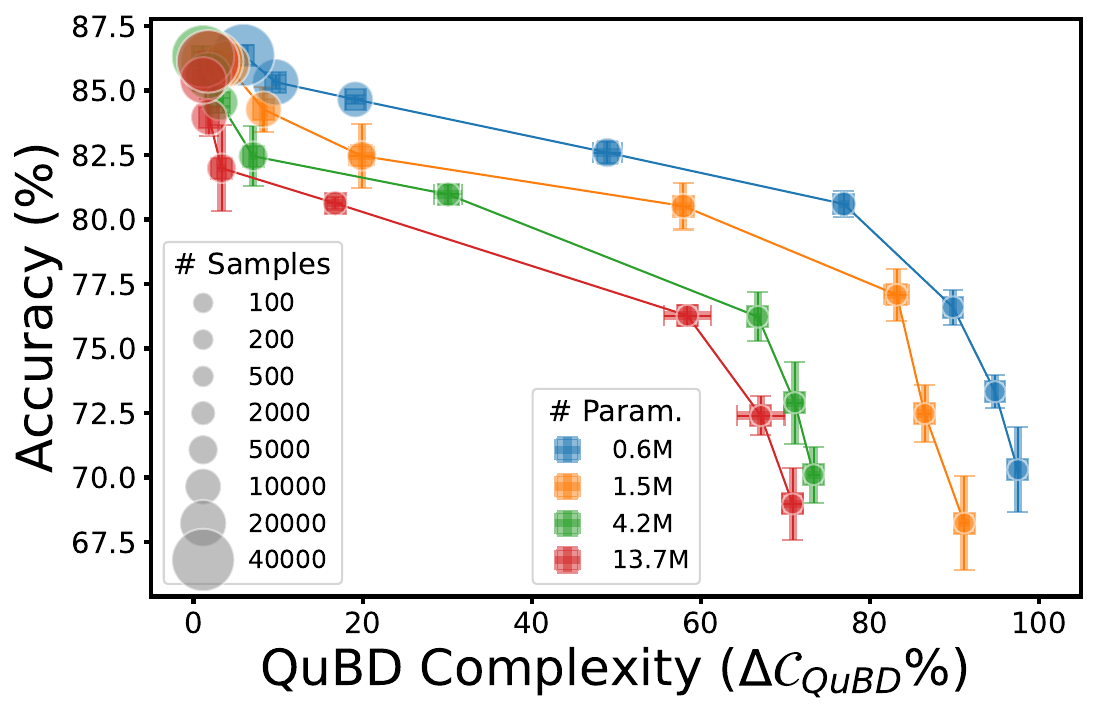}
    \includegraphics[width=0.325\linewidth]{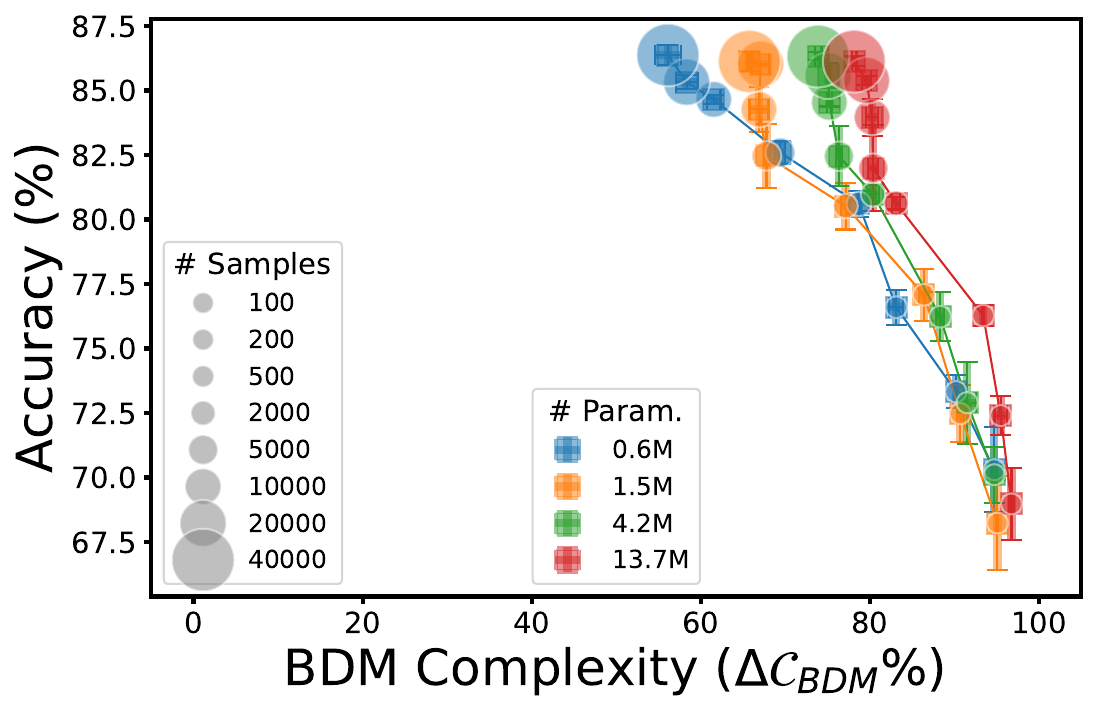}
    \includegraphics[width=0.325\linewidth]{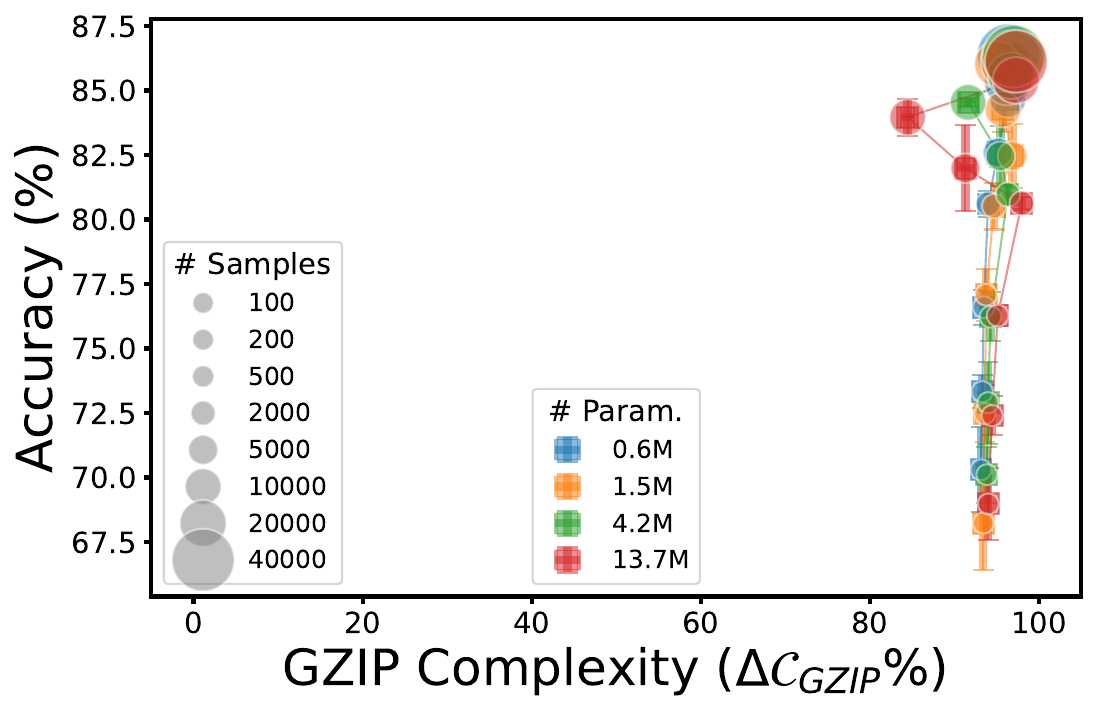}
    \caption{Complexity of trained models compared to random initialization when using QuBD, BDM, and GZIP-based estimations, measured as $\Delta \Ccal$ based on  Eq.~\ref{eq:delta_c}. Four MLPs with different numbers of trainable parameters ($0.6M-13.7M$) are trained with increasing number of training data points and validated on the same test set. QuBD complexity tracks the connection between learning and algorithmic complexity better than BDM or GZIP.
    }
    \label{fig:data-budget}
\end{figure}

In Fig.~\ref{fig:data-budget} \textbf{(left)}, we present clear empirical evidence that the algorithmic complexity of neural networks measured using QuBD (Sec.~\ref {sec:method}) decreases with training. Multi-layered perceptrons (MLPs) with varying numbers of trainable parameters, ranging from $0.6M$ to $13.7M$, are trained on the Fashion-MNIST dataset \citep{xiao2017fashion} with increasing number of training samples and evaluated on the same validation set. For each model, the test accuracy improves with more training data, and the normalized QuBD complexity, $\Delta \Ccal_\text{QuBD}$, decreases accordingly, indicating that the improvements in test performance correlate with reductions in algorithmic complexity. This effect is most pronounced for the largest model, whose QuBD complexity reduces to approximately 5\% of its value at initialization. To investigate whether there is an effect of the number of symbols $d$ exposed for varying object sizes, for instance, in Fig.~\ref{fig:complexity}, we compute the Spearman correlation of the relationship between model size and $\Delta \Ccal_{\text{QuBD}}$ and find a weak relationship ($|\rho| < 0.2$) for all bit-planes $\ell=0,\ldots,7$, illustrated in Appx~\ref{app:d_effect} (Fig.~\ref{fig:d_effect}). This suggests that the ordering is stable across model sizes, consistent with the saturation effect characterized for increasing object sizes in Sec.~\ref{sec:saturation_effect}.

Fig.~\ref{fig:data-budget} \textbf{(center)} and Fig.~\ref{fig:data-budget} \textbf{(right)} report two commonly used complexity measures for comparison: BDM with weight binarization~\citep {sakabe2025evaluating,bakhtiarifard2026momos}, and raw weight compression using GZIP. As clearly seen in these plots, and further validated in Appx.~\ref{app:bitplane_shannon} (Fig.~\ref{fig:qubd_shannon_gzip}), neither measure correlates meaningfully with learning, in contrast to QuBD. The results shown here are when using 8 bit-planes for QuBD. Additional experimental details are provided in Appx.~\ref{app:pareto}.

\begin{figure}[h]
    \centering
    \includegraphics[height=0.16\linewidth]{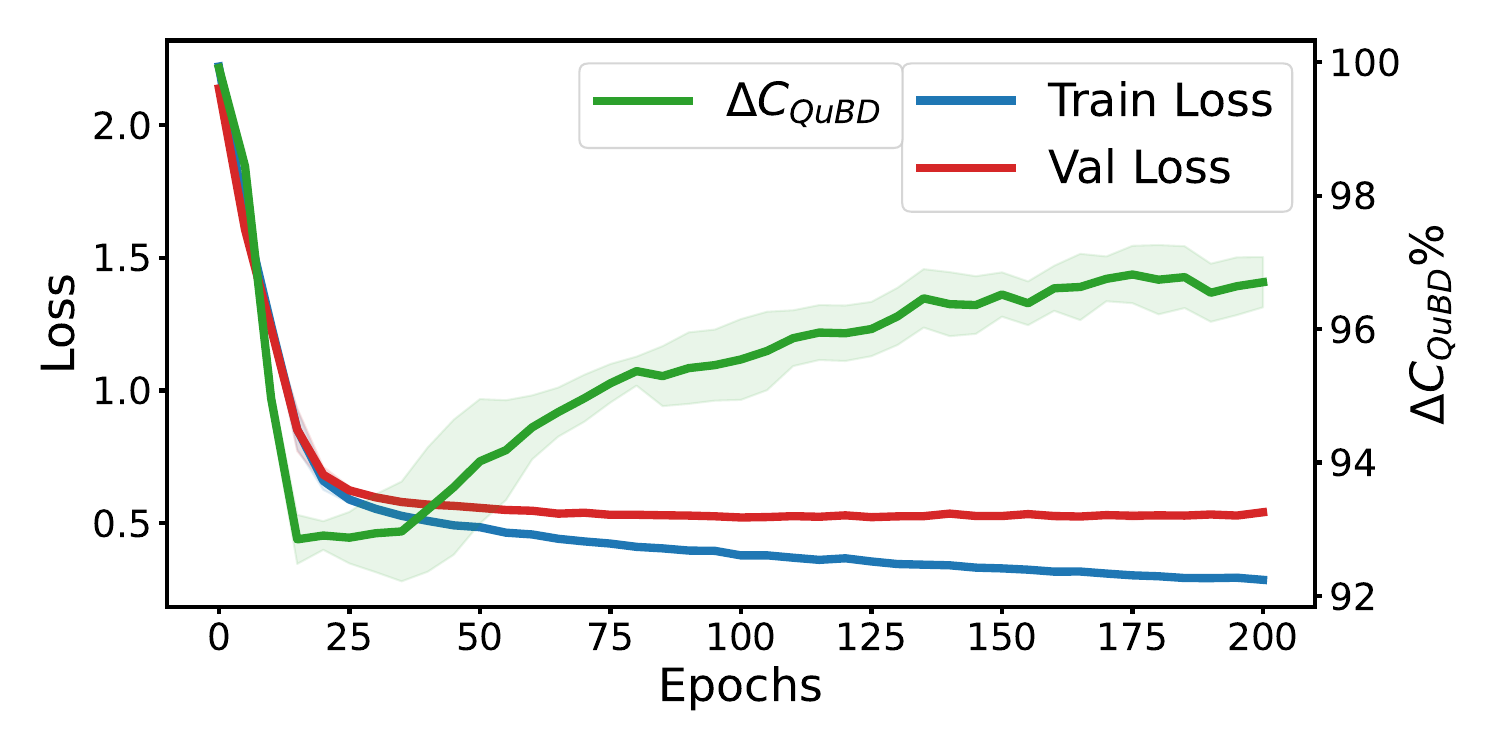}
    \includegraphics[height=0.16\linewidth]{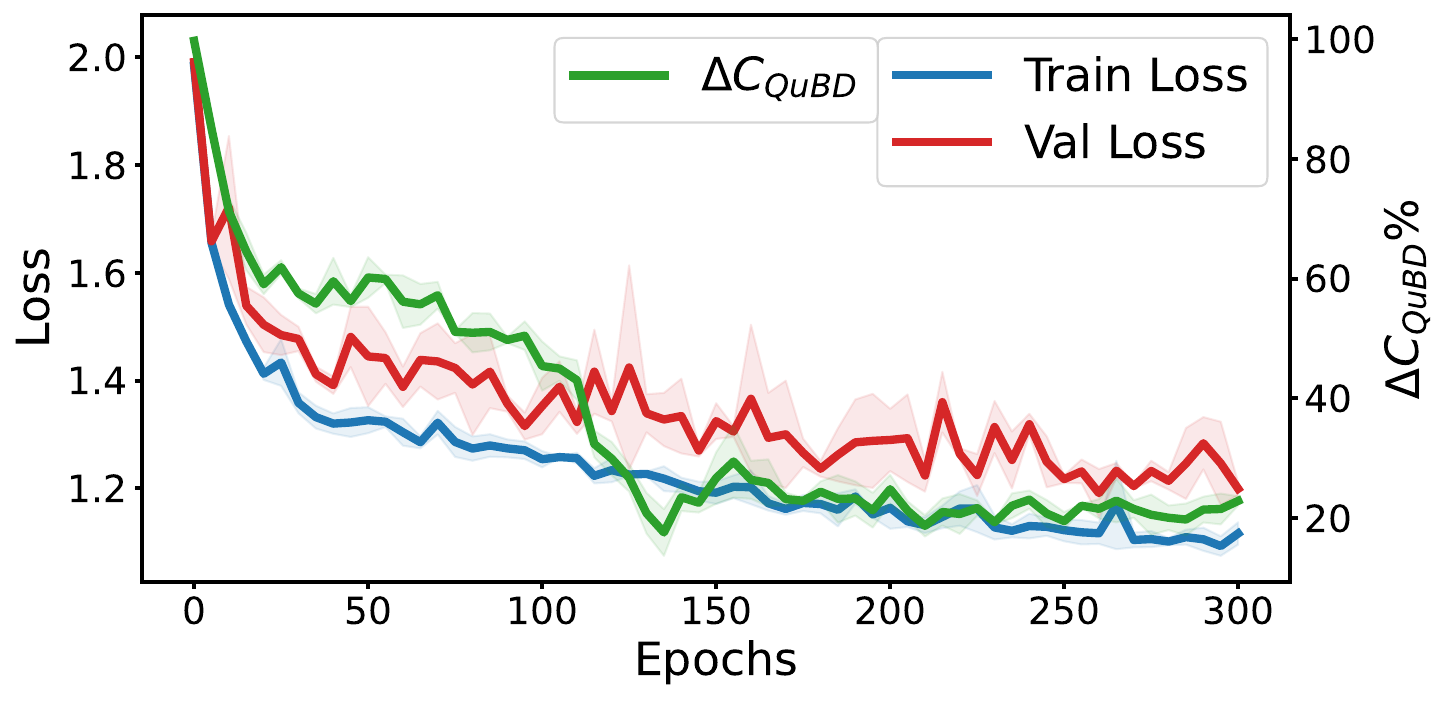}
    \includegraphics[height=0.16\linewidth]{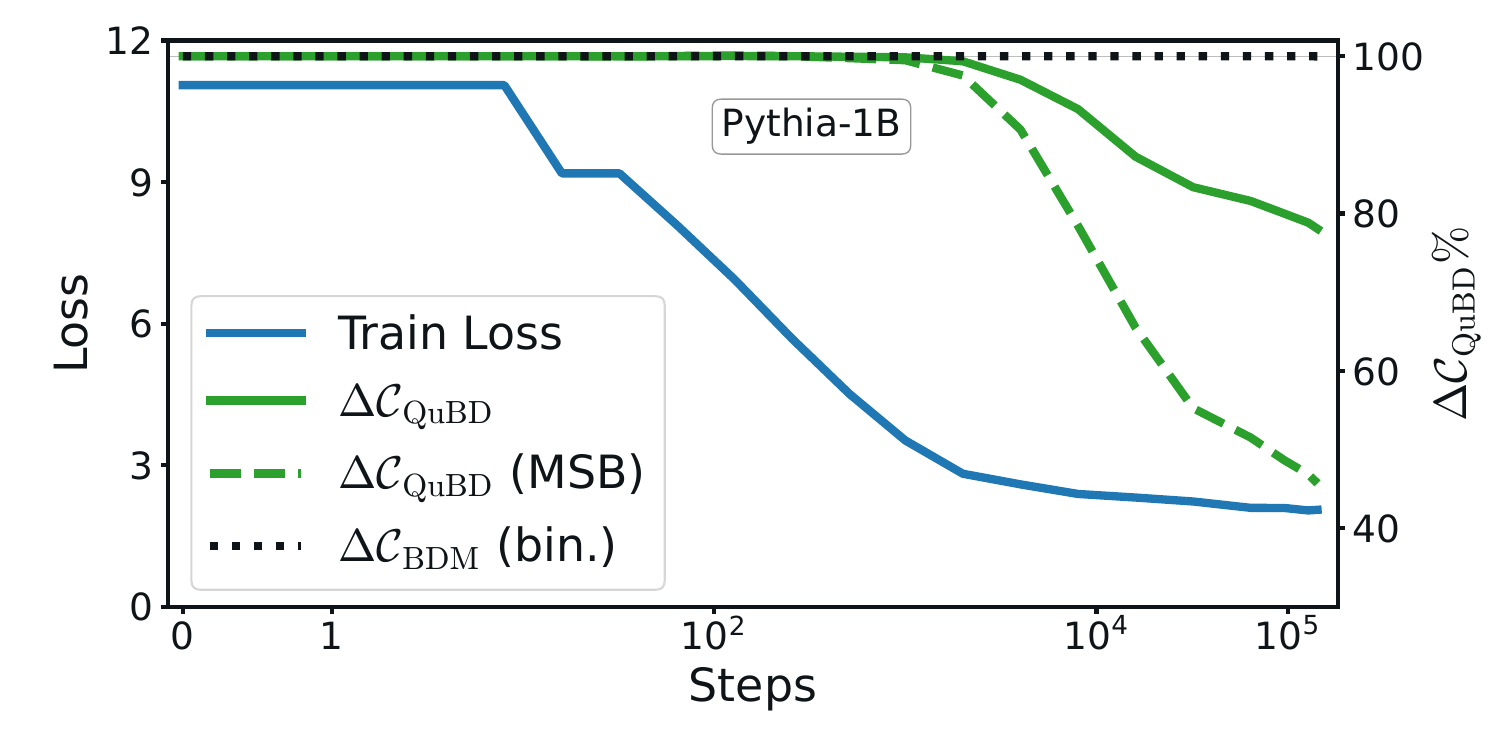}
    \caption{Learning curves along with $\Delta\Ccal_\text{QuBD}$ evolution for MLP \textbf{(left)}, Tiny-ViT \textbf{(center)} and Pythia-1B \textbf{(right)}. Loss and $\Delta \Ccal_\text{QuBD}$ for MLP and Tiny-ViT are reported over three random repeats. 
    }
    \label{fig:training}
\end{figure}
\paragraph{Algorithmic Complexity Tracks Generalization.} Fig.~\ref{fig:training} \textbf{(right)} probes into the evolution of algorithmic complexity reduction during training for an MLP trained on Fashion-MNIST~\citep{xiao2017fashion} and Tiny-ViT~\citep{tinyvit} trained on CIFAR-10~\citep{krizhevsky2009learning}. In both cases, $\Delta\Ccal_\text{QuBD}$ tracks the generalization performance. The most interesting observation is in Fig.~\ref{fig:training} \textbf{(left)}, where the algorithmic complexity increases as the network starts to overfit. As overfitting is generally attributed to memorization, this suggests that the network is no longer simplifying its underlying algorithmic structure, leading to an increase in algorithmic complexity. For the Tiny-ViT case, QuBD complexity follows the learning curves, indicating that learning continues to reduce the algorithmic complexity of networks. Additional experimental details are provided in Appx.~\ref{app:learning_curve}.  We further study whether the relation between training dynamics and QuBD complexity extends to frontier AI models, using the Pythia-1B large language model (LLM), which was trained using The Pile dataset comprising about 300B tokens~\citep{biderman2023pythia}. Similar to~Fig.~\ref{fig:training} \textbf{(left)}, we observe that the QuBD complexity decreases, and the network continues simplification after the loss plateaus, illustrated in Appx.~\ref{app:pythia} (Fig.~\ref{fig:pythia}).

\paragraph{Algorithmic Complexity during Grokking.} 
Grokking~\citep{power2022grokking} is a phenomenon observed in over-parameterized neural networks, where the networks generalize well past the point of overfitting, defying conventional machine learning principles~\citep{Belkin_2019}. This raises the question: {\em What happens to the algorithmic complexity of neural networks in such regimes?}

\begin{wrapfigure}{r}{0.35\textwidth}
\vspace{-0.4cm}
\centering
\includegraphics[width=0.95\linewidth]{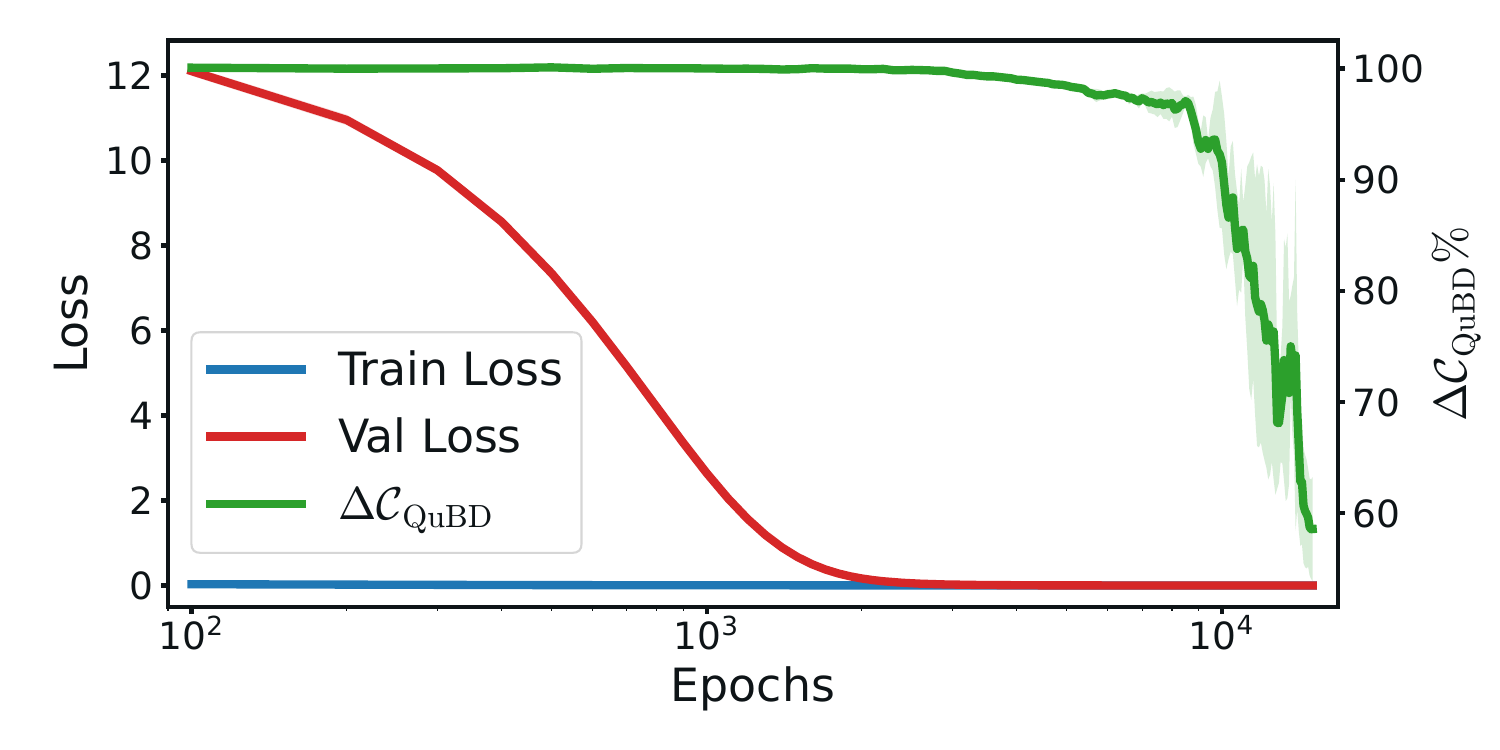}
\vspace{-0.2cm}
\caption{Grokking for the Modulo operator for $P=97$.}
\label{fig:grokking_modulo}
\vspace{-0.4cm}
\end{wrapfigure}
To study this, we set up grokking experiments for the modulo operator as described in~\cite{power2022grokking}, and measure the QuBD complexity during training, as shown in Fig.~\ref{fig:grokking_modulo}. The network reaches zero training error within the first few epochs, whereas the validation loss drops to zero after about 1000 epochs, consistent with the findings reported in~\cite{power2022grokking} and other works. The algorithmic complexity, measured as the reduction in QuBD complexity, remains high during the high-validation-loss regime and decreases when the validation loss has reached zero. It continues to decrease even after the validation loss is zero, indicating further algorithmic simplification of the neural network. This is in contrast to the overfitting example demonstrated in Fig.~\ref{fig:training} \textbf{(left)}.

\begin{wrapfigure}{r}{0.35\textwidth}
\vspace{-0.4cm}
\centering
\includegraphics[width=0.95\linewidth]{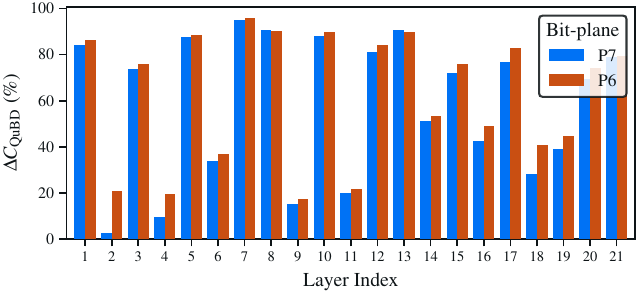}
\vspace{-0.2cm}
\caption{QuBD complexity ratio $\Delta\Ccal_\text{QuBD}$ per layer for ResNet-18 for the two MSB-planes (P7, P6).}
\label{fig:qubd_reduction_per_layer_resnet18}
\vspace{-0.4cm}
\end{wrapfigure}
\paragraph{Algorithmic Information Resides in Higher Bit-Planes.} When bit-planes are independent, QuBD complexity offers tighter estimates on ideal prefix KCS complexity of non-binary objects compared to BDM, due to Prop.~\ref{thm:res_decomp} and Prop.~\ref{thm:add_bitplane}. Adding bit-planes captures additional algorithmic information, but saturates beyond a few. This is illustrated in Fig.~\ref{fig:complexity}, where we show the relative complexity reduction achieved by the QuBD method per bit-plane. For the higher bit-planes, in particular the MSB, $\Delta\Ccal_\text{QuBD}$ is lower for all 100 models with $1M$ to $100M$ trainable parameters, whereas the lower bit-planes show no reduction in complexity. Additional experimental details can be found in Appx.~\ref{app:per_layer}. 

To use QuBD for models with high-dimensional weight tensors, e.g., 4D convolutional layers, the weight tensors must first be preprocessed into 2D matrices prior to estimation, as described in Appx.~\ref{app:weight_preprocess}. Using this, we examine $\Delta\Ccal_\text{QuBD}$ per layer of ResNet-18 (Fig.~\ref{fig:qubd_reduction_per_layer_resnet18}) and ResNet-50  (Appx.~\ref{app:qubd_reduction_per_layer_resnet50}) for the two MSB-planes (P7, P6). It is well established that different layers in a neural network have different importance, and that a uniform level of pruning or quantization can hamper performance. This variation across layers is clearly reflected in the QuBD-complexity reduction for both ResNet models~\citep{wenshoj2025codeq}. While there is no clear trend, some inner layers appear to exhibit significant complexity reduction. This could also explain why these inner layers are better suited for compression than input or output layers.
\looseness=-1

We have seen that lower bit-planes have less algorithmic information, leading to the question: {\em If the lower bit-planes contain less algorithmic information, can these planes be discarded?}
\looseness=-1

\paragraph{QuBD Complexity as a Diagnostic for Compression.} 
\begin{wrapfigure}{r}{0.5\textwidth}
\vspace{-0.3cm}
\centering
\includegraphics[width=\linewidth]{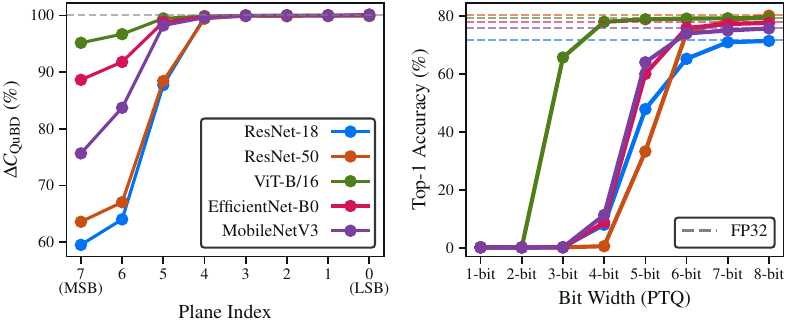}
\vspace{-0.5cm}
\caption{QuBD complexity $\Delta\Ccal_\text{QuBD}$ per bit-plane \textbf{(left)} and top-1 PTQ accuracy on ImageNet \textbf{(right)} for five pretrained models under 1-to-8-bit quantization. Dashed lines indicate FP32 baseline accuracy.}
\label{fig:qubd_ptq}
\vspace{-0.25cm}
\end{wrapfigure}
To test this hypothesis, we measure the QuBD complexity of five full precision (FP32) models: ResNet-18, ResNet-50, ViT, EfficientNet, and MobileNet, trained on ImageNet \citep{deng2009imagenet} using 16 bit-planes. We then perform post-training quantization (PTQ) on the same models with weight precision ranging from 1-bit to 8-bit and evaluate the quantized models on the ImageNet validation set. The first observation in Fig.~\ref{fig:qubd_ptq} \textbf{(left)} is that $\Delta\Ccal_\text{QuBD}\approx 100\%$ for the lower bit-planes, indicating they are close to random and might not contain any useful algorithmic information. Notably, meaningful complexity reduction is concentrated in the top five bit-planes. This correlates with the PTQ performance in Fig.~\ref{fig:qubd_ptq} \textbf{(right)}, which does not show large performance drops after 4-to-6-bit quantization, suggesting that the remaining bit-planes can be discarded without a performance penalty. Measuring KCS complexity using the QuBD method can serve as a diagnostic to determine PTQ quantization levels (Fig.~\ref{fig:qubd_ptq}) or which layers to prune (Fig. \ref{fig:qubd_reduction_per_layer_resnet18}). Additional experimental details are provided in Appx.~\ref{app:ptq}. 
\looseness=-1

\paragraph{Influence of Quantization on QuBD Complexity.}\label{sec:quant_influence}
Estimating the QuBD complexity entails a quantization step as described in Sec.~\ref{sec:method}.
One concern about this transformation is whether the resulting KCS complexity is related to the actual model weights or whether quantization has an adverse impact. We verify this by comparing standard FP32 training with quantization-aware training (QAT) at 32-bit precision, and measuring the QuBD complexity using 8 bit-planes. 
\begin{wraptable}{r}{0.35\textwidth}
\centering
\caption{Comparison of FP32 and QAT with 32 bits.}
\footnotesize
\begin{tabular}{@{}lrr@{}}
\toprule
{\bf Precision} & {\bf Accuracy} & $\Delta\Ccal_\text{QuBD}$\\ \midrule
 FP32 & 86.13 ± 0.46 & 95.51 ± 0.90 \\
 QAT-32 & 86.22 ± 0.44 & 96.39 ± 0.29\\
 \bottomrule
\end{tabular}
\label{tab:qat}
\vspace{-0.3cm}
\end{wraptable}
Tab.~\ref{tab:qat} shows that there is no significant difference between the original model and the QAT model, both in terms of accuracy and QuBD complexity reduction. These are key differences compared to other methods that rely on binarizing the weights of an FP32 network before estimating KCS complexity, as the two objects (FP32 and binary versions) may not be comparable in those settings. This is further exemplified in Appx.~\ref{app:pythia} where the QuBD MSB-plane does not coincide with the sign-pattern found when binarizing.

 We further verify estimation robustness to the effect of reshaping in Appx.~\ref{app:weight_preprocess}, and the weight distribution incurred by percentile-clipping in Appx.~\ref{app:percentile_bdm_vs_qubd} (Tab.~\ref{tab:clipped_symmetric_quantization}). This can change the structures exposed to the CTM table. We find, however, negligible impact for both effects, with the trends observed in Fig.~\ref{fig:data-budget} persisting across all model scales and sample budgets. Importantly, no other assumptions on architecture, layer shapes, or the distribution, for instance, its symmetry, of the trained weights or their extrema are imposed.

\paragraph{Limitations.} We aggregate the algorithmic complexity per bit-plane in QuBD, which assumes independence across bit-planes. This might not be true as there could be mutual algorithmic information shared across bit-planes, which implies QuBD can be overestimating the underlying algorithmic complexity. While QuBD provides a tighter upper-bound estimate on the ideal prefix complexity than BDM, accounting for the mutual information between bit-planes could yield even tighter KCS complexity estimations.  
Further, we demonstrated that QuBD can identify redundant bit-planes and that weight structures can be pruned based on QuBD complexity, suggesting its potential as a practical diagnostic for both PTQ level selection and pruning. A natural extension is to validate these applications within thorough compression experiments and benchmark them against established methods. For instance, measuring saturation in individual bit-plane exposures in practice.
As we find consistent evidence for complexity reduction estimated by QuBD, the checkpoints used for the larger Pythia models consider just one epoch, so they do not extend to the overfitting regime and over random initializations, but do coincide with the grokking experiments.

\section{Conclusions}
\label{sec:conc}
Several attempts have been made to explain learning in deep neural networks. We contribute a novel view based on algorithmic information theory, building on some of the recent works~\citep{sakabe2025evaluating,bakhtiarifard2026momos}. We operationalize KCS complexity estimation using the QuBD method. To our knowledge, this is the first scalable method to reliably estimate KCS complexity of large, non-binary objects. We have shown theoretically that the QuBD method under bit-plane independence yields a tighter upper-bound estimate on ideal prefix KCS complexity compared to binarization. Using this algorithmic complexity analysis tool, we have shown that learning in deep neural networks induces structure, or equivalently, reduces algorithmic complexity. We have shown, across a wide variety of experiments, that QuBD complexity correlates with generalization and phenomena such as grokking. We have demonstrated that higher bit-planes carry more algorithmic structure than lower bit-planes. This can be used to identify redundant model capacity, for example, by choosing appropriate PTQ levels. Algorithmic information theory is a powerful paradigm that has been creatively applied to prove Darwinian evolution~\citep{chaitin2013proving}, and we believe it can also be useful for explaining deep learning.

\begin{ack}
All authors acknowledge funding from the European Union’s Horizon Europe Research and Innovation Action program under grant agreements No. 101070284, No. 101070408 and No. 101189771. PB acknowledges funding from Stibofonden, Knud Højgaards Fond, William Demant Fonden, and Familien Hede Nielsens Fond. 
SW and RS acknowledge funding from the Independent Research Fund Denmark (DFF) under grant agreement No. 4307-00143B. The authors thank members of \href{https://saintslab.github.io/}{SAINTS Lab} for valuable discussions.
\end{ack}







\bibliographystyle{abbrvnat}
\bibliography{references}

@inproceedings{biderman2023pythia,
  title={{Pythia: A Suite for Analyzing Large Language Models Across Training and Scaling}},
  author={Biderman, Stella and Schoelkopf, Hailey and Anthony, Quentin Gregory and Bradley, Herbie and O’Brien, Kyle and Hallahan, Eric and Khan, Mohammad Aflah and Purohit, Shivanshu and Prashanth, USVSN Sai and Raff, Edward and others},
  booktitle={International Conference on Machine Learning (ICML)},
  OPTpages={2397--2430},
  year={2023},
  OPTorganization={PMLR}
}

@book{li2008introduction,
  title = {{An Introduction to Kolmogorov Complexity and its Applications}},
  author={Li, Ming and Vit{\'a}nyi, Paul and others},
  OPTvolume={3},
  year={2008},
  publisher={Springer}
}

@article{paszke2019pytorch,
  title = {{Pytorch: An Imperative Style, High-Performance Deep Learning Library}},
  author={Paszke, Adam and Gross, Sam and Massa, Francisco and Lerer, Adam and Bradbury, James and Chanan, Gregory and Killeen, Trevor and Lin, Zeming and Gimelshein, Natalia and Antiga, Luca and others},
  journal={Neural Information Processing Systems (NeurIPS)},
  OPTvolume={32},
  year={2019}
}

@article{bakhtiarifard2026momos,
  title = {{Algorithmic Simplification of Neural Networks with Mosaic-of-Motifs}},
  author={Bakhtiarifard, Pedram and Chen, Tong and Wensh{\o}j, Jonathan and Dam, Erik B and Selvan, Raghavendra},
  journal={arXiv preprint arXiv:2602.14896},
  year={2026}
}

@article{Belkin_2019,
   title = {{Reconciling Modern Machine Learning Practice and the Classical Bias–Variance Trade-Off}},
   OPTvolume={116},
   OPTISSN={1091-6490},
   OPTurl={http://dx.doi.org/10.1073/pnas.1903070116},
   OPTDOI={10.1073/pnas.1903070116},
   OPTnumber={32},
   journal={National Academy of Sciences},
   publisher={National Academy of Sciences},
   author={Belkin, Mikhail and Hsu, Daniel and Ma, Siyuan and Mandal, Soumik},
   year={2019},
   OPTmonth=jul, OPTpages={15849–15854} }

@book{chaitin2013proving,
  title = {{Proving Darwin: Making Biology Mathematical}},
  author={Chaitin, Gregory J.},
  year={2013},
  publisher={Vintage}
}

@inproceedings{neyshabur2017exploringgeneralizationdeeplearning,
  title = {{Exploring Generalization in Deep Learning}},
  author    = {Neyshabur, Behnam and Bhojanapalli, Srinadh and McAllester, David and Srebro, Nathan},
  booktitle = {{Neural Information Processing Systems (NeurIPS)}},
  OPTpages     = {5949--5958},
  year      = {2017},
  OPTurl       = {https://proceedings.neurips.cc/paper/2017/hash/10ce03a1ed01077e3e289f3e53c72813-Abstract.html}
}

@misc{pybdmSoftware,
  title     = {{PyBDM: Python interface to the Block Decomposition Method (0.1.0)}},
  author    = "Talaga, Szymon and Tsampourakis, Kostas",
  publisher = "Zenodo",
  year      =  2024
}

@book{cover1999elements,
  title = {{Elements of Information Theory}},
  author={Cover, Thomas M},
  year={1999},
  publisher={Wiley--Interscience}
}

@article{chaitin1966length,
  title = {{On the Length of Programs for Computing Finite Binary Sequences}},
  author={Chaitin, Gregory J},
  journal={Journal of the Association for Computing Machinery},
  OPTvolume={13},
  OPTnumber={4},
  OPTpages={547--569},
  year={1966},
  OPTpublisher={Association for Computing Machinery}
}

@article{delahaye2012numerical,
  title = {{Numerical Evaluation of Algorithmic Complexity for Short Strings: A Glance Into the Innermost Structure of Randomness}},
  author={Delahaye, Jean-Paul and Zenil, Hector},
  journal={Applied Mathematics and Computation},
  OPTvolume={219},
  OPTnumber={1},
  OPTpages={63--77},
  year={2012},
  publisher={Elsevier}
}

@article{mcmillan2003two,
  title = {{Two Inequalities Implied by Unique Decipherability}},
  author={McMillan, Brockway},
  journal={IRE Transactions on Information Theory},
  OPTvolume={2},
  OPTnumber={4},
  OPTpages={115--116},
  year={2003},
  publisher={IEEE}
}

@phdthesis{kraft1949device,
  title = {{A Device for Quantizing, Grouping, and Coding Amplitude-Modulated Pulses}},
  author={Kraft, Leon Gordon},
  year={1949},
  school={Massachusetts Institute of Technology}
}

@article{solomonoff1964formal,
  title = {{A Formal Theory of Inductive Inference. Part I}},
  author={Solomonoff, Ray J},
  journal={Information and Control},
  OPTvolume={7},
  OPTnumber={1},
  OPTpages={1--22},
  year={1964},
  publisher={Elsevier}
}

@inproceedings{
goldblum2024position,
title = {{Position: The No Free Lunch Theorem, Kolmogorov Complexity, and the Role of Inductive Biases in Machine Learning}},
author = {Goldblum, Micah and Finzi, Marc Anton and Rowan, Keefer and Wilson, Andrew Gordon},
booktitle = {{International Conference on Machine Learning (ICML)}},
OPTpublisher =    {PMLR},
year={2024},
OPTurl={https://openreview.net/forum?id=EaJ7nqJ2Fa}
}

@article{zenil2015twodimensionalkolmogorovcomplexityvalidation,
  title = {{Two-Dimensional Kolmogorov Complexity and Validation of the Coding Theorem Method by Compressibility}},
  author={Zenil, Hector and Soler-Toscano, Fernando and Delahaye, Jean-Paul and Gauvrit, Nicolas},
  journal={PeerJ Computer Science},
  OPTvolume={1},
  OPTpages={e23},
  year={2015},
  publisher={PeerJ Inc.}
}

@article{gong2014compressing,
  title = {{Compressing Deep Convolutional Networks using Vector Quantization}},
  author={Gong, Yunchao and Liu, Liu and Yang, Ming and Bourdev, Lubomir},
  journal={arXiv preprint arXiv:1412.6115},
  year={2014}
}

@misc{rw2019timm,
  author = {Ross Wightman},
  title = {{PyTorch Image Models}},
  year = {2019},
  publisher = {GitHub},
  journal = {GitHub repository},
  OPTdoi = {10.5281/zenodo.4414861},
  howpublished = {\url{https://github.com/rwightman/pytorch-image-models}}
}

@misc{Seward1996bzip2,
  author       = {Seward, Julian},
  title = {{bzip2 and libbzip2: A Program and Library for Data Compression}},
  year         = {1996},
  OPTnote         = {Version 1.0.8},
  url          = {http://www.bzip.org}
}

@article{gzipAuthors,
  author = {Gailly, Jean-Loup and Adler, Mark},
  title  = {{Gnu gzip}},
  OPTnote   = {Created as a free replacement for {Unix} compress},
  year   = {1992},
  journal = {GNU Operating System},
  url    = {https://www.gnu.org/}
}

@article{xu2017information,
  title = {{Information-Theoretic Analysis of Generalization Capability of Learning Algorithms}},
  author={Xu, Aolin and Raginsky, Maxim},
  journal={Neural Information Processing Systems (NeurIPS)},
  OPTvolume={30},
  year={2017}
}

@article{achille2018emergenceinvariancedisentanglementdeep,
  title = {{Emergence of Invariance and Disentanglement in Deep Representations}},
  author={Achille, Alessandro and Soatto, Stefano},
  journal={Journal of Machine Learning Research},
  OPTvolume={19},
  OPTnumber={50},
  OPTpages={1--34},
  year={2018}
}

@inproceedings{dinh2017sharpminimageneralizedeep,
  title = {{Sharp Minima Can Generalize For Deep Nets}},
  author    = {Dinh, Laurent and Pascanu, Razvan and Bengio, Samy and Bengio, Yoshua},
  booktitle = {{International Conference on Machine Learning (ICML)}},
  OPTpages     = {1019--1028},
  year      = {2017},
  OPTpublisher = {PMLR},
  OPTurl       = {https://proceedings.mlr.press/v70/dinh17b.html}
}

@article{xiao2017fashion,
  title = {{Fashion-MNIST: a Novel Image Dataset for Benchmarking Machine Learning Algorithms}},
  author={Xiao, Han and Rasul, Kashif and Vollgraf, Roland},
  journal={arXiv preprint arXiv:1708.07747},
  year={2017}
}

@misc{krizhevsky2009learning,
  title = {{Learning Multiple Layers of Features from Tiny Images}},
  author      = {Krizhevsky, Alex},
  institution = {University of Toronto},
  year        = {2009},
  OPTurl         = {https://www.cs.toronto.edu/~kriz/learning-features-2009-TR.pdf}
}

@inproceedings{deng2009imagenet,
author    = {Jia Deng and Wei Dong and Richard Socher and Li-Jia Li and Kai Li and Li Fei-Fei},
title     = {{ImageNet: A Large-Scale Hierarchical Image Database}},
booktitle = {{IEEE Conference on Computer Vision and Pattern Recognition}},
year      = {2009}
}

@article{rethink_zhang,
author = {Zhang, Chiyuan and Bengio, Samy and Hardt, Moritz and Recht, Benjamin and Vinyals, Oriol},
year = {2021},
pages = {},
title = {{Understanding Deep Learning (Still) Requires Rethinking Generalization}},
journal = {Communications of the Association for Computing Machinery},
OPTdoi = {10.1145/3446776}
}

@article{bartlett2002rademacher,
  title = {{Rademacher and Gaussian Complexities: Risk Bounds and Structural Results}},
  author={Bartlett, Peter L and Mendelson, Shahar},
  journal={Journal of Machine Learning Research},
  OPTvolume={3},
  OPTnumber={Nov},
  OPTpages={463--482},
  year={2002}
}

@article{vc99,
author = {Vapnik, V. N.},
title = {{An Overview of Statistical Learning Theory}},
year = {1999},
OPTissue_date = {September 1999},
publisher = {IEEE Press},
OPTvolume = {10},
OPTnumber = {5},
OPTissn = {1045-9227},
OPTurl = {https://doi.org/10.1109/72.788640},
OPTdoi = {10.1109/72.788640},
journal = {IEEE Transactions on Neural Networks},
OPTmonth = sep,
OPTpages = {988–999},
OPTnumpages = {12}
}

@article{shwartz2017opening,
  title = {{Opening the Black Box of Deep Neural Networks via Information}},
  author={Shwartz-Ziv, Ravid and Tishby, Naftali},
  journal={arXiv preprint arXiv:1703.00810},
  year={2017}
}

@article{han2016deepcompressioncompressingdeep,
  title={{Deep Compression: Compressing Deep Neural Networks with Pruning, Trained Quantization and Huffman Coding}},
  author={Han, Song and Mao, Huizi and Dally, William J},
  journal={International Conference on Learning Representations (ICLR)},
  year={2016}
}

@article{lotfi2022pac,
  title={{PAC-Bayes Compression Bounds so Tight that they can Explain Generalization}},
  author={Lotfi, Sanae and Finzi, Marc and Kapoor, Sanyam and Potapczynski, Andres and Goldblum, Micah and Wilson, Andrew G},
  journal={Neural Information Processing Systems (NeurIPS)},
  OPTvolume={35},
  OPTpages={31459--31473},
  year={2022}
}

@article{turing1936computable,
  title={{On Computable Numbers, with an Application to the Entscheidungsproblem}},
  author={Turing, Alan Mathison},
  journal={In Proceedings of the London Mathematical Society},
  OPTvolume={58},
  OPTnumber={345-363},
  OPTpages={5},
  year={1936},
}

@inproceedings{frankle2019lotterytickethypothesisfinding,
  title = {{The Lottery Ticket Hypothesis: Finding Sparse, Trainable Neural Networks}},
  author    = {Frankle, Jonathan and Carbin, Michael},
  booktitle = {{International Conference on Learning Representations (ICLR)}},
  year      = {2019},
  OPTurl       = {https://openreview.net/forum?id=rJl-b3RcF7}
}

@inproceedings{kingma2015adam,
  title = {{Adam: A Method for Stochastic Optimization}},
  author={Kingma, Diederik P and Ba, Jimmy},  booktitle = {{International Conference on Learning Representations (ICLR)}},
  year      = {2015}
}

@article{sakabe2025evaluating,
  title={{Binarized Neural Networks Converge Toward Algorithmic Simplicity: Empirical Support for the Learning-as-Compression Hypothesis}},
  author={Sakabe, Eduardo Y and Abrah{\~a}o, Felipe S and Sim{\~o}es, Alexandre and Colombini, Esther and Costa, Paula and Gudwin, Ricardo and Zenil, Hector},
  journal={Frontiers in Computational Neuroscience},
  OPTvolume={20},
  OPTpages={1791546},
  year={2026},
  publisher={Frontiers Media SA}
}

@inproceedings{
wilson2025deep,
title={{Position: Deep Learning is Not So Mysterious or Different}},
author={Andrew Gordon Wilson},
booktitle={International Conference on Machine Learning (ICML) Position Paper Track},
year={2025},
OPTurl={https://openreview.net/forum?id=42Au7FoD8F}
}

@article{chaitin1977algorithmic,
  title = {{Algorithmic Information Theory}},
  author={Chaitin, Gregory J},
  journal={IBM Journal of Research and Development},
  OPTvolume={21},
  OPTnumber={4},
  OPTpages={350--359},
  year={1977},
  publisher={IBM}
}

@inproceedings{sevilla2022compute,
  title = {{Compute Trends Across Three Eras of Machine Learning}},
  author={Sevilla, Jaime and Heim, Lennart and Ho, Anson and Besiroglu, Tamay and Hobbhahn, Marius and Villalobos, Pablo},
  booktitle = {{International Joint Conference on Neural Networks}},
  year={2022},
  organization={IEEE}
}

@inproceedings{strubell2019energy,
  title={Energy and Policy Considerations for Deep Learning in {NLP}},
  author={Strubell, Emma and Ganesh, Ananya and McCallum, Andrew},
  booktitle = {{Annual Meeting of the Association for Computational Linguistics}},
  year={2019}
}

@article{bartoldson2023compute,
  title = {{Compute-Efficient Deep Learning: Algorithmic Trends and Opportunities}},
  author={Bartoldson, Brian R and Kailkhura, Bhavya and Blalock, Davis},
  journal={Journal of Machine Learning Research},
  year={2023}
}

@article{nagel2021white,
  title = {{A White Paper on Neural Network Quantization}},
  author={Nagel, Markus and Fournarakis, Marios and Amjad, Rana Ali and Bondarenko, Yelysei and Van Baalen, Mart and Blankevoort, Tijmen},
  journal={arXiv preprint arXiv:2106.08295},
  year={2021}
}

@inproceedings{anthony2020carbontracker,
  title={{Carbontracker: Tracking and Predicting the Carbon Footprint of Training Deep Learning Models}},
  author={Lasse F. Wolff {Anthony} and Benjamin Kanding and Raghavendra {Selvan}},
  booktitle = {{International Conference on Machine Learning (ICML) Workshop on Challenges in Deploying and Monitoring Machine Learning Systems}},
  year={2020}}

@incollection{gholami2021surveyquantizationmethodsefficient,
  title = {{A Survey of Quantization Methods for Efficient Neural Network Inference}},
  author={Gholami, Amir and Kim, Sehoon and Dong, Zhen and Yao, Zhewei and Mahoney, Michael W and Keutzer, Kurt},
  booktitle={Low-Power Computer Vision},
  OPTpages={291--326},
  year={2022},
  publisher={Chapman \& Hall/CRC}
}

@article{martin2025setolsemiempiricaltheorydeep,
  title = {{SETOL: A Semi-Empirical Theory of (Deep) Learning}},
  author={Martin, Charles H and Hinrichs, Christopher},
  journal={arXiv preprint arXiv:2507.17912},
  year={2025}
}

@book{zenil2020algorithmic,
  title = {{Algorithmic Information Dynamics}},
  author={Zenil, Hector and Kiani, Narsis and Abrah{\~a}o, Felipe and Tegn{\'e}r, Jesper},
  year={2020},
  publisher={Scholarpedia}
}

@article{wenshoj2025codeq,
  title = {{CoDeQ: End-to-End Joint Model Compression with Dead-Zone Quantizer for High-Sparsity and Low-Precision Networks}},
  author={Wensh{\o}j, Jonathan and Chen, Tong and Pepin, Bob and Selvan, Raghavendra},
  journal={arXiv preprint arXiv:2512.12981},
  year={2025}
}

@article{soler2014calculating,
  title = {{Calculating Kolmogorov Complexity from the Output Frequency Distributions of small Turing Machines}},
  author={Soler-Toscano, Fernando and Zenil, Hector and Delahaye, Jean-Paul and Gauvrit, Nicolas},
  journal={Public Library of Science ONE},
  OPTvolume={9},
  OPTnumber={5},
  OPTpages={e96223},
  year={2014},
  publisher={Public Library of Science San Francisco, USA}
}

@article{zenil2018decomposition,
  title = {{A Decomposition Method for Global Evaluation of Shannon Entropy and Local Estimations of Algorithmic Complexity}},
  author={Zenil, Hector and Hern{\'a}ndez-Orozco, Santiago and Kiani, Narsis A and Soler-Toscano, Fernando and Rueda-Toicen, Antonio and Tegn{\'e}r, Jesper},
  journal={Entropy},
  OPTvolume={20},
  OPTnumber={8},
  OPTpages={605},
  year={2018},
  publisher={MDPI}
}

@techreport{solomonoff1960preliminary,
  title = {{A Preliminary Report on a General Theory of Inductive Inference}},
  author={Solomonoff, Ray J},
  year={1960},
  institution={Zator Company Cambridge, Massachusetts}
}

@book{davis1982computability,
  title = {{Computability \& Unsolvability}},
  author={Davis, Martin},
  year={1982},
  publisher={Courier Corporation}
}

@article{kolmogorov1965three,
  title = {{Three Approaches to the Quantitative Definition of Information}},
  author={Kolmogorov, Andrei N},
  journal={Problems of Information Transmission},
  OPTvolume={1},
  OPTnumber={1},
  OPTpages={1--7},
  year={1965}
}

@article{levin1974laws,
  title = {{Laws of Information Conservation (Nongrowth) and Aspects of the Foundation of Probability Theory}},
  author={Levin, Leonid Anatolevich},
  journal={Problemy Peredachi Informatsii},
  OPTvolume={10},
  OPTnumber={3},
  OPTpages={30--35},
  year={1974},
  publisher={Russian Academy of Sciences, Branch of Informatics, Computer Equipment and~…}
}

@inproceedings{
ullrich2017soft,
title = {{Soft Weight-Sharing for Neural Network Compression}},
author={Karen Ullrich and Edward Meeds and Max Welling},
booktitle = {{International Conference on Learning Representations (ICLR)}},
year={2017},
OPTurl={https://openreview.net/forum?id=HJGwcKclx}
}

@inproceedings{hinton2014distilling,
  title = {{Distilling the Knowledge in a Neural Network}},
  author={Geoffrey Hinton and Oriol Vinyals and Jeff Dean},
  booktitle = {{Deep Learning and Representation Learning Workshop in Conjunction with NeurIPS}},
  year={2014}
}

@inproceedings{tinyvit,
  title = {{TinyViT: Fast Pretraining Distillation for Small Vision Transformers}},
  author    = {Wu, Kan and Zhang, Jinnian and Peng, Houwen and Liu, Mengchen and Xiao, Bin and Fu, Jianlong and Yuan, Lu},
  booktitle = {{European Conference on Computer Vision (ECCV)}},
  year      = {2022}
}

@inproceedings{valledeep,
  title = {{Deep Learning Generalizes because the Parameter-Function Map is Biased Towards Simple Functions}},
  author={Valle-Perez, Guillermo and Camargo, Chico Q and Louis, Ard A},
  year = 2019,
  booktitle = {{International Conference on Learning Representations (ICLR)}}
}

@article{Liu2024DeltaGzip,
author = {Liu, Tao and Simine, Lena},
title = {{DeltaGzip: Computing Biopolymer–Ligand Binding Affinity via Kolmogorov Complexity and Lossless Compression}},
journal = {Journal of Chemical Information and Modeling},
OPTvolume = {64},
OPTnumber = {14},
OPTpages = {5617-5623},
year = {2024},
OPTdoi = {10.1021/acs.jcim.4c00461},
    OPTnote ={PMID: 38980667},

OPTurl = { 
    
        https://doi.org/10.1021/acs.jcim.4c00461
    
    

},
eprint = { 
    
        https://doi.org/10.1021/acs.jcim.4c00461
    
    

}
}

@article{Machado2021,
  author = {Machado, J. A. Tenreiro and Rocha-Neves, J. M. and Azevedo, Filipe and Andrade, J. P.},
  title = {{Advances in the Computational Analysis of SARS-COV2 Genome}},
  journal = {Nonlinear Dynamics},
  year = {2021},
  OPTvolume = {106},
  OPTnumber = {2},
  OPTpages = {1525--1555},
  OPTmonth = {Oct},
  OPTdoi = {10.1007/s11071-021-06836-y},
  OPTurl = {https://doi.org/10.1007/s11071-021-06836-y},
  OPTissn = {1573-269X}
}

@article{Alzahrani2024,
    OPTdoi = {10.1371/journal.pone.0301806},
    author = {Alzahrani, Alaa},
    journal = {PLOS ONE},
    publisher = {Public Library of Science},
    title = {{Utility of Kolmogorov Complexity Measures: Analysis of L2 Groups and L1 Backgrounds}},
    year = {2024},
    OPTmonth = {04},
    OPTvolume = {19},
    OPTurl = {https://doi.org/10.1371/journal.pone.0301806},
    OPTpages = {1-25},
    OPTnumber = {4},

}

@inproceedings{Rigau2007,
author = {Rigau, Jaume and Feixas, Miquel and Sbert, Mateu},
year = {2007},
title = {{Conceptualizing Birkhoff's Aesthetic Measure Using Shannon Entropy and Kolmogorov Complexity.}},
OPTdoi = {10.2312/COMPAESTH/COMPAESTH07/105-112},
booktitle = {{Third Eurographics Conference on Computational Aesthetics in Graphics, Visualization and Imaging}}
}

@article{Peptenatu2022,
  title = {{Kolmogorov Compression Complexity May Differentiate Different Schools of Orthodox Iconography}},
  author={Peptenatu, Daniel and Andronache, Ion and Ahammer, Helmut and Taylor, Richard and Liritzis, Ioannis and Radulovic, Marko and Ciobanu, Bogdan and Burcea, Marin and Perc, Matjaz and Pham, Tuan D and others},
  journal={Scientific Reports},
  OPTvolume={12},
  OPTnumber={1},
  OPTpages={10743},
  year={2022},
  publisher={Nature Publishing Group UK London}
}

@inproceedings{chang1999new,
  title={{A New Scheme of Progressive Image Transmission Based on Bit-Plane Method}},
  author={Chang, Chin-Chen and Shiue, Fuh-Chou and Chen, Tung-Shou},
  booktitle={Fifth Asia-Pacific Conference on... and Fourth Optoelectronics and Communications Conference on Communications},
  OPTvolume={2},
  OPTpages={892--895},
  year={1999},
  organization={IEEE}
}

@article{nhan2025improvement,
  title = {{Improvement of Spiking Neural Network with Bit Planes and Color Models}},
  author={Nhan, Luu Trong and Duong, Luu Trung and Nam, Pham Ngoc and Thang, Truong Cong},
  journal={IEEE Access},
  OPTvolume={13},
  OPTpages={198607--198622},
  year={2025},
  publisher={IEEE}
}

@article{tishby2000information,
  title = {{The Information Bottleneck Method}},
  author={Tishby, Naftali and Pereira, Fernando C and Bialek, William},
  journal={In Proceedings at Allerton Conference on Communication},
  year={2000}
}

@inproceedings{power2022grokking,
  title = {{Grokking: Generalization Beyond Overfitting on Small Algorithmic Datasets}},
  author={Power, Alethea and Burda, Yuri and Edwards, Harri and Babuschkin, Igor and Misra, Vedant},
  booktitle={{International Conference on Learning Representations (ICLR) Workshop on Mathematical Reasoning in General Artificial Intelligence}},
  year={2021}
}

\appendix

\clearpage
\appendix

\section{Additional Details on Algorithmic Information Theory}\label{app:background}
For completeness, the definitions introduced in the main text are restated here, with additional detail provided where necessary. Kolmogorov-Chaitin-Solomonoff (KCS) complexity
\citep{kolmogorov1965three,chaitin1966length,solomonoff1960preliminary}
formalizes the principle that structured objects can be described more concisely than random objects.
For a finite object $\vw\in\{0,1\}^*$ and a program $\mathrm{p}$, this relationship is expressed as: \begin{equation}\label{eq:app_kc}
\Ccal(\vw)
:=
\min \{|\mathrm p|:U(\mathrm p)=\vw \}.
\end{equation}
Here, $\mathrm p$ is executed on a fixed universal prefix Turing machine
$U$. In general, Eq.~\eqref{eq:app_kc} is not computable because determining whether all shorter programs produce $\vw$ and halt, fail, or run indefinitely is undecidable. This is known as the \emph{halting problem}~\citep{turing1936computable,davis1982computability}.

A related measure to the KCS complexity of an object is its algorithmic probability~\citep{solomonoff1964formal},
which assigns a probability to each possible output produced by a program:
\begin{equation}\label{eq:app_ap}
\Prob_U(\vw)
:=
\sum_{\mathrm p:U(\mathrm p)=\vw}2^{-|\mathrm p|}
\le 1 .
\end{equation}
Here $\Prob_U(\vw)$ is the probability that a random program produces $\vw$ and halts. The last inequality follows from the Kraft–McMillan inequality
\citep{kraft1949device,mcmillan2003two}, since the halting programs form a
prefix code, i.e., no program is the start of another program.

To help gain intuition into approaches for estimating KCS using algorithmic probability, consider an object $\vw$ produced by a program $\mathrm{p}$, and suppose there are $m$ distinct programs of length $|\mathrm{p}|$ such that $U(\mathrm{p})=\vw$. Then $\Prob_U (\vw) \ge m \cdot 2^{-|\mathrm{p}|}$ by Eq.~\ref{eq:app_ap} and $\Ccal(\vw) \le |\mathrm{p}|$, expressing the upper bound on its KCS complexity. For instance, if $|\mathrm{p}| = 4$, each such program contributes $2^{-4} = 1/16$, so $\Prob_U (\vw) \ge m/16$. Thus, the existence of one short program gives a lower bound on the algorithmic probability of $\vw$. More generally, a larger output probability suggests the existence of shorter (program) descriptions.

An important caveat is that prefix programs are self-delimiting, meaning their description also contains information about where they end. Therefore, estimating complexity from a decomposed view of an object is not the same as estimating the entire object directly. In particular, suppose the partition $\Pi(\vw)=(\w_1,\ldots,\w_m)$ divides the object $\vw$ into $m$ ordered blocks $\w_j$. Then,
\begin{equation}
\Ccal(\w_1,\ldots,\w_m) = \sum_{j=1}^m \Ccal(\w_j) - I_\Pi(\w_1,\ldots,\w_m) +
O(\log_2 m),
\end{equation}
where $\Ccal(\w_j)$ denotes the prefix KCS complexity of $\w_j$ and $I_\Pi$ denotes the algorithmic information shared across blocks in the partition $\Pi$. An estimate that sums local complexities, therefore, ignores part of this shared structure, and the $O(\log_2 m)$ term accounts for the overhead needed to delimit and reconstruct the ordered decomposition of $\vw$. This induces a representation and partition error, which is why both matter. Some useful identities about these errors are summarized in Appx.~\ref{app:kc_facts}.

The Coding Theorem~\citep{levin1974laws} connects the output probability of an object to its complexity by
\begin{equation}
\Ccal(\vw) = -\log_2 \Prob_U (\vw) + O(1) ,
\end{equation}
where $O(1)$ is a constant factor depending only on $U$, e.g., changing to a machine that uses different instructions adds a constant cost irrespective of $\vw$. The central idea behind the Coding Theorem, that objects that appear often as outputs are of low complexity, motivated the Coding Theorem Method (CTM) for numerically estimating KCS complexity via large-scale simulations of small Turing machines~\citep{delahaye2012numerical, soler2014calculating, zenil2015twodimensionalkolmogorovcomplexityvalidation}.

\paragraph{Simulating Turing Machines.}
Fix integers $n,k,T\in\N$ and let $\cM_{n,k}^T$ denote the set of all
$n$-state, $k$-symbol Turing machines under a fixed encoding, each run for at
most $T$ steps. Since $n,k,T$ are fixed, $\cM_{n,k}^T$ is finite, and all
machines in the class can, in principle, be simulated. Let
$\rho_T(A)=\vw\in\cW_T$ denote the output of $A$ whenever $A$ halts within
$T$ steps. CTM estimates the algorithmic probability of $\vw$ by its empirical
frequency
\begin{equation}\label{eq:app_ctm_dist}
\mathrm{Dist}_{n,k}^T(\vw)
=
\frac{|\{A\in\cM{n,k}^T:\rho_T(A)=\vw\}|}
{|\{A\in\cM_{n,k}^T:A\text{ halts within }T\text{ steps}\}|},
\end{equation}
and the corresponding complexity estimate is  $\Ccal_{\CTM}(\vw) = -\log_2 \mathrm{Dist}_{n,k}^T(\vw)+O(1)$.
For larger objects, however, exhaustive simulation is not tractable, since the number of machines grows exponentially with $n$ and $k$.

\paragraph{Block Decomposition Method.}
To extend the CTM to larger objects, \cite{zenil2018decomposition} introduced the Block Decomposition Method (BDM), which partitions the object into contiguous blocks and sums the complexity estimates of its smaller structures. Let $\vw \in \Sigma^{H \times W}$ be a large object over a finite alphabet $\Sigma$, and let $\Pi$ be a lossless partition operator such that
\begin{equation}\label{eq:partition}
\Pi(\vw) = (\w_1, \dots, \w_m), \qquad \vw = \w_1 \otimes \cdots \otimes \w_m,
\end{equation}
where each $\w_j \in \Sigma^{h \times w}$ is a contiguous block of size $h \times w$, and $\otimes$ denotes the reconstruction of $\vw$ by concatenating the ordered blocks. If a block appears several times, its complexity is accounted for only once, with a logarithmic term for its frequency; the idea is that repeated structures need not be described from scratch. Let $\Pi_u(\vw) \subseteq \Pi(\vw)$ denote the set of distinct blocks in the partition, and let $n_{\w_j}$ denote the number of occurrences of $\w_j \in \vw$. The KCS complexity is then computed as
\begin{equation}
\Ccal_\mathrm{BDM}(\vw) =
\sum_{\w_j \in \Pi_u(\vw)}
\Ccal_{\mathrm{CTM}}(\w_j) + \log_2 n_{\w_j},
\end{equation}
where the logarithmic term accounts for the complexity of repeating a block. Hence, repeated structures have diminishing complexity beyond their first occurrences. For example, consider the objects $\vw_1 = aa \otimes aa \otimes ab \otimes aa$, and $\vw_2 = aa \otimes ab \otimes ba \otimes bb$. Assuming $\Ccal_{\mathrm{CTM}}(aa)= \Ccal_{\mathrm{CTM}}(ab)= \Ccal_{\mathrm{CTM}}(ba)= \Ccal_{\mathrm{CTM}}(bb)= 1$, then $\mathrm{BDM}(\vw_1) = 1 + \log_2 3 + 1 \approx 3.6$ and $\mathrm{BDM} (\vw_2) = 4$.

\paragraph{Limitations.}
In practice, the accuracy of $\Ccal_\BDM$ is limited by the finite CTM tables it relies on, where values are computed from a restricted machine class $\cM_{n,k}^T$ with fixed $n$, $k$, and $T$. As a result, differences in
estimates may reflect either structural differences or limitations of the
underlying CTM approximation. In particular, partition choices are not
invariant in practice. Although they represent the same object up to an additive constant, different block shapes, e.g., $(2\times1)$ versus $(1\times2)$, produce different local blocks and multiplicities, and therefore different $\Ccal_\BDM$ estimates. In the unbounded setting, such differences are bounded up to an additive constant for fixed lossless partition functions~\citep{kolmogorov1965three}.

The most extensive CTM table is the $5$-state, $2$-symbol
enumeration of~\citet{soler2014calculating}, which considered over
$2.6\times10^{13}$ machines and required $450$ CPU days. As the basis for the
$\BDM$ approximation, \citet{zenil2018decomposition} showed that such finite
CTM values still provide meaningful local rankings of short patterns and
capture simple algorithmic structure\footnote{We refer to
\citet{zenil2018decomposition} for further practical considerations, including
overlapping versus non-overlapping partitions and their effect on the
estimates.}. However, extending CTM tables is computationally infeasible,
especially in high-dimensional settings such as neural networks with large
parameter spaces and high-precision values. In practice, objects are therefore
first transformed to finite representations, and then decomposed into structures that are closer to what the available CTM tables can support.

\paragraph{Compression Proxies.}\label{app:background_compression}
A common approach is to use compression algorithms as proxies for KCS
complexity. The better an object compresses, the less complex we expect it to be. Such approximations have been used to reason about structure and how
generative patterns emerge. For instance, GZIP~\citep{gzipAuthors} has been
used to approximate Gibbs entropy in molecular systems~\citep{Liu2024DeltaGzip}
and as a language-agnostic metric for distinguishing native from second-language writing~\citep{Alzahrani2024}. Bzip2~\citep{Seward1996bzip2} has been
used as a proxy for KCS complexity to estimate evolutionary distance between
SARS-CoV-2 sequences~\citep{Machado2021}. In image settings, JPG and PNG have
been used as practical proxies for categorizing paintings by style and
artist~\citep{Rigau2007}, and for distinguishing schools of orthodox
iconography~\citep{Peptenatu2022}. These methods are useful diagnostics, but
they are not CTM or BDM estimators, since they rely on compressor-specific
coding schemes rather than algorithmic probability tables.

\clearpage
\section{QuBD Method Details}

\subsection{Quantization}\label{app:quantizer}
Neural network weights are real-valued. For $b$-bit scalar quantization, we map
$\mathbf{w}\in\mathbb{R}^d$ to integer codes $\mathbf{q}\in\{0,\dots,q_{\max}\}^d$,
where $q_{\max}=2^b-1$, by applying a quantizer elementwise. Let
\[
\w_{\min}:=\min_i \w_i,\qquad
\w_{\max}:=\max_i \w_i,\qquad
\Delta_\w:=\w_{\max}-\w_{\min},\qquad
s:=\frac{\Delta_\w}{q_{\max}}.
\]
Define the (min--max) uniform affine quantizer  $\mathbf{q}=Q_b(\vw),$ where $Q_b:\mathbb{R}^d\to\{0,\dots,q_{\max}\}^d$ by the elementwise operation:
\begin{equation}
    q_i=
    \begin{cases}
    \mathrm{clip}\!\left(\left\lfloor\dfrac{\w_i-\w_{\min}}{s}\right\rceil,\;0,\;q_{\max}\right), & \Delta_\w \ge\varepsilon,\\[6pt]
    0, & \Delta_\w<\varepsilon,
    \end{cases}
    \label{eq:uniform_quantizer}
\end{equation}

where $\lfloor\cdot\rceil$ rounds to nearest integer, $\mathrm{clip}(\cdot,0,q_{\max})$ clamps to the valid range,
and $\varepsilon>0$ is a small constant for numerical stability.
The corresponding dequantizer $D_b:\{0,\dots,q_{\max}\}^d\to\mathbb{R}^d$ is
\begin{equation}
\widehat{\mathbf{w}}=D_b(\mathbf{q})=s\,\mathbf{q}+\w_{\min}\mathbf{1}.
    \label{eq:dequantize_uniform}    
\end{equation}
When performing affine quantization, we adhere to the standard practice and set $w_{\min}$ and $w_{\max}$ to the 0.05th and 99.95th percentiles of each weight tensor. That is, we keep the central 99.9\% of values and clip the 0.1\% most extreme values, which may include outliers.  

\subsection{Example of Bit-Plane Decomposition}\label{appendix:bitplane_example}
Consider an object $\vw = [\w_1,\w_2]$ with $\w_1=4$ and $\w_2=13$. And for intuition, let the target precision be $b^\star=4$. Then the
lossless quantized target is
\begin{equation}
    \vw^\star = Q_{b^\star}(\vw) = (4,13) = (0100_2,1101_2).
\end{equation}
The corresponding bit-planes are
\begin{equation}
    \boldsymbol{\beta}_3^\star = (0,1), \qquad \boldsymbol{\beta}_2^\star = (1,1), \qquad \boldsymbol{\beta}_1^\star = (0,0), \qquad \boldsymbol{\beta}_0^\star = (0,1).
\end{equation}
If we keep only $b=2$ most significant bit-planes, then
\begin{equation}
    \mathbf{q}_2 = \floor{\frac{\vw^\star}{2^{4-2}}} = \floor{\frac{(4,13)}{4}} = (1,3) = (01_2,11_2).
\end{equation}
So only the first two bit-planes, $\boldsymbol{\beta}_3^\star$ and $\boldsymbol{\beta}_2^\star$, are
retained. The omitted lower-bit residual is
\begin{equation}
    \delta_2 = \vw^\star - 2^{4-2} q_2 = (4,13) - 4(1,3) = (0,1) = (00_2,01_2).
\end{equation}
Then we have
\begin{equation}
    \vw^\star = (4,13) = (0100_2,1101_2) = 2^{4-2}(01_2,11_2) + (00_2,01_2).
\end{equation}
In this example, $\w_1=4$ is already determined by its two retained MSB
bit-planes, while $\w_2=13$ still requires the lower residual bit-planes to
recover the symbol. Thus estimating $\Ccal(\mathbf{q}_2)$ ignores the
information stored in $\delta_2$.

\subsection{Weight Tensor Preprocessing}\label{app:weight_preprocess}
Prior to computing the QuBD complexity of a network, each of its weight tensors must be represented as a 2D binary matrix suitable for BDM estimation. This requires two steps: reshaping the weight tensor into a 2D matrix, and extracting its bit-planes as described in Sec. \ref{sec:method}. The reshaping strategy matters because BDM is sensitive to the spatial adjacency of elements, and different reshaping choices expose different local structures to the CTM table. BDM is, however, invariant to transposing a matrix, so its orientation does not matter.

\paragraph{Weight Reshaping.} Complexity is computed over the weight tensors of linear and convolutional layers, while biases, normalization parameters, and other parameters such as positional embeddings are excluded. Fully connected layers have weight tensors of shape $[C_\text{out}, C_\text{in}]$ and are already 2D, requiring no reshaping. Tensors with more than two dimensions (e.g., 4D convolutional kernels of shape $[C_\text{out}, C_\text{in}, H, W]$) are flattened along all but the first dimension, yielding a matrix of shape $[C_\text{out}, C_\text{in} \cdot H \cdot W]$. Each row then corresponds to one filter, with the weights of each kernel kept contiguous, and a $1 \times 1$ convolution yields the same matrix as a fully connected layer. Matrices with fewer than 4 rows or columns after reshaping are excluded, as they cannot be partitioned into blocks of the size supported by the precomputed CTM tables used in this work ($\pi = 4 \times 4$; \citealt{pybdmSoftware}). Matrices whose dimensions are not multiples of 4 are partitioned without padding: rows and columns that do not fill a complete block are ignored, which for the 100 models in Fig. \ref{fig:complexity} affects less than 0.1\% of the weights overall and at most 1.4\% in any single model. 
\looseness=-1

\paragraph{Robustness to Reshaping.} The flattening of convolutional weights described above is one of several possible conventions. Since BDM is sensitive to which weights are adjacent, we test whether this choice affects the QuBD complexity ratio by replacing it with four alternative reshapings of the same weights:
\begin{enumerate}[label=(\roman*)]
    \item \emph{Input-channel rows}, $[C_\text{in}, C_\text{out} \cdot H \cdot W]$: each row contains the kernels of one input channel instead of one filter.
    \item \emph{Channels-last}, $[C_\text{out}, H \cdot W \cdot C_\text{in}]$: the input channel varies fastest within a row, so that the same kernel position of different input channels is adjacent, and the weights of a single kernel are spread $C_\text{in}$ entries apart.
    \item \emph{Closest-to-square}: the flattened weights are reshaped into the most square matrix $[r, c]$ with $r \cdot c = C_\text{out} \cdot C_\text{in} \cdot H \cdot W$, ignoring the semantics of the axes.
    \item \emph{Kernel mosaic}, $[C_\text{out} \cdot H, C_\text{in} \cdot W]$: each kernel is kept as its $H \times W$ patch, arranged in a grid of filters and input channels.
\end{enumerate}

Quantization depends only on the weight values, so all reshapings yield identical quantized values and differ only in their arrangement. Fully connected layers are unaffected by the reshapings. We compute the model-level QuBD complexity ratio $\Delta \Ccal_\text{QuBD}$ per bit-plane for the five models of Fig.~\ref{fig:qubd_ptq} (8-bit quantization, clipping at the 99.9th percentile) and report the largest absolute deviation from the default over all bit-planes in Tab.~\ref{tab:reshape_ablation}.

\begin{table}[h]
\centering
\caption{Largest absolute deviation (percentage points) of $\Delta \Ccal_\text{QuBD}$ from the default reshaping over all 8 bit-planes. Dash denotes that the reshaping leaves every layer of the model unchanged.}
\label{tab:reshape_ablation}
\begin{tabular}{lccccc}
\toprule
Reshaping & ResNet-18 & ResNet-50 & ViT-B/16 & EfficientNet-B0 & MobileNetV3 \\
\midrule
(i) Input-channel rows & 0.2 & 0.1 & -- & -- & -- \\
(ii) Channels-last & 12.5 & 6.7 & 0.2 & 0.0 & 0.0 \\
(iii) Closest-to-square & 0.4 & 0.1 & -- & 0.3 & 0.1 \\
(iv) Kernel mosaic & 1.6 & 0.8 & 0.2 & 0.5 & 0.1 \\
\bottomrule
\end{tabular}
\end{table}

Reshapings that keep the weights of each kernel together, (i), (iii), and (iv), change $\Delta \Ccal_\text{QuBD}$ by at most 1.6 percentage points across all models and bit-planes. The only reshaping with a notable effect is channels-last (ii), which is also the only one that separates the weights of a kernel: in ResNet-18 and ResNet-50 it increases $\Delta \Ccal_\text{QuBD}$ on the most significant bit-planes by up to 12.5 and 6.7 percentage points, respectively. This indicates that the structure captured by QuBD in convolutional layers lies within the spatial dimensions of the kernels, which the default reshaping preserves. For ViT-B/16, EfficientNet-B0 and MobileNetV3, the choice of reshaping has almost no effect, as nearly all of their weights lie in fully connected, $1 \times 1$ or depthwise layers, where (i), (ii) and (iv) coincide with the default and (iii) changes $\Delta \Ccal_\text{QuBD}$ by at most 0.3 percentage points.

\clearpage \newpage
\section{QuBD Complexity}
\subsection{Information Identities for KCS Complexity}\label{app:kc_facts}
We list here some useful KCS complexity identities and briefly interpret them to build further intuition. We refer to~\cite{li2008introduction} for additional details.

\textbf{Algorithmic Mutual Information.} For finite binary strings $x,y,z \in \{0,1\}^*$, the algorithmic information in
$x$ about $y$ (and not) conditioned on $z$ is 
\begin{equation}\label{eq:algo_mutual_info}
    I(x:y) := \Ccal(y) - \Ccal(y \mid x), \quad I(x:y \mid z) := \Ccal(y \mid z) - \Ccal(y \mid x, z) .
\end{equation}

\textbf{Chain Rule for KCS Complexity.} For finite binary strings $x,y \in \{0,1\}^*$,
\begin{equation}\label{eq:chain_rule}
\begin{aligned}
    \Ccal(x,y)
    &= \Ccal(x) + \Ccal\bigl(y \mid x, \Ccal(x)\bigr) + O(1) \\
    &\equiv\;
    \Ccal(x) + \Ccal\bigl(y \mid x, \Ccal(x)\bigr)
    = \Ccal(y) + \Ccal\bigl(x \mid y, \Ccal(y)\bigr) + O(1).
\end{aligned}
\end{equation}

\textbf{Symmetry of Information.} For finite binary strings $x,y \in \{0,1\}^*$,
\begin{equation}\label{eq:symmetry}
\begin{aligned}
    I\bigl((x,\Ccal(x)) : y\bigr)
    &= I\bigl((y,\Ccal(y)) : x\bigr) + O(1) \\
    &\equiv\;
    \Ccal(y) - \Ccal\bigl(y \mid x, \Ccal(x)\bigr)
    = \Ccal(x) - \Ccal\bigl(x \mid y, \Ccal(y)\bigr) + O(1).
\end{aligned}
\end{equation}

Eq.~\eqref{eq:algo_mutual_info} isolates the part of $y$ that can be reused from $x$, and the residual $\Ccal(y \mid x)$, which remains irreducible. But note that $I(x:y) \neq I(y:x)$ in general, because the description of one object inside another is not self-delimiting, i.e., from $x$ alone we cannot know where the description of $x$ ends inside a description of the pair $(x,y)$. This is apparent in Eq.~\eqref{eq:chain_rule} where the additional conditioning on $\Ccal(x)$ in $\Ccal(y \mid x, \Ccal(x))$ provides the information to distinguish the objects. Conveniently, Eq.~\eqref{eq:symmetry} shows exactly when mutual information between objects is in fact symmetric.

However, it can be clearer to reason about the symmetry of information assuming \emph{composite} objects, e.g., $(x,y)$ are self-delimiting. We therefore omit additional conditionals in our analysis, which instead incurs an error in information of $O(\log b^\star d)$ for objects with bit length $O(b^\star d)$.

\subsection{Proof of Proposition~\ref{thm:res_decomp} (Residual Decomposition)}
\label{proof:res_decomp}
By construction $\vw^\star = 2^{(b^\star-b)} \mathbf{q}_b + \delta_b$, so the target object $\vw^\star$ and the composite object $(\mathbf{q}_b,\delta_b)$ describe each other up to a fixed additive constant:
\begin{equation}
    \Ccal(\vw^\star) = \Ccal(\mathbf{q}_b,\delta_b) + O(1).
\end{equation}
Applying the chain rule from Eq.~\eqref{eq:chain_rule} to $(\q_b,\delta_b)$ gives
\begin{equation}
    \Ccal(\q_b,\delta_b) = \Ccal(\q_b) + \Ccal(\delta_b \mid \q_b, \Ccal(\q_b)) + O(1).
\end{equation}
Since the objects here have bit length $O(b^\star d)$, the components are self-delimiting up to a logarithmic term, and therefore
\begin{equation}
\begin{aligned}
     \Ccal(\q_b,\delta_b) &= \Ccal(\q_b) + \Ccal(\delta_b \mid \q_b) + O(\log(b^\star d)) \\
    \implies \quad
    \Ccal(\vw^\star) &= \Ccal(\q_b) + \Ccal(\delta_b \mid \q_b) + O(\log(b^\star d)) . 
\end{aligned}
\end{equation}
Letting $R_b = \Ccal(\delta_b \mid \q_{b})$ gives Eq.~\eqref{thm:res_decomp_eq1}. That is, if $\vw^\star$ is estimated only from the retained prefix $\q_b$, then the estimate misses the residual information contained in the discarded lower bits, up to the logarithmic term.

It remains to bound the size of $R_b$. By the construction, each residual element $\delta_{b,i}$ can take one of $2^{(b^\star-b)}$ possible values, namely from $\{0,\ldots,2^{(b^\star-b)}-1\}$. Therefore, it requires at most $b^\star-b$ bits to encode. Over all $d$ elements, the bit length of $\delta_b$ is at most $(b^\star-b)d$. A conditional description of $\delta_b$ given $\q_b$ therefore cannot be longer than this bit length up to an additive constant, so we get
\begin{equation}
\begin{aligned}
    0 \le \Ccal(\delta_b \mid \q_b) \le (b^\star-b)d + O(1) \\
    \implies \quad
    0 \le R_b \le (b^\star-b)d + O(1) .
\end{aligned}    
\end{equation}
Thus, the residual estimation gap is bounded above in the number of discarded bits, $O((b^\star-b)d)$. Finally, if $b=b^\star$, then no lower bit-planes are omitted, so $\delta_b=0$ and $R_b = \Ccal(0 \mid \q_b) = O(1)$. \hfill \qedsymbol

\subsection{Proof of Proposition~\ref{thm:add_bitplane} (Bit-Plane Residual Loss)}
\label{proof:add_bitplane}
By definition of the MSB-prefixes, for each $i \in \{1,\ldots,d\}$,
\begin{equation}
    \q_{b+1,i}
    = \floor{\frac{\w_i^\star}{2^{(b^\star-b)-1}}}
    = 2\floor{\frac{\w_i^\star}{2^{(b^\star-b)}}} + \beta^\star_{i,(b^\star-b)-1}
    = 2\q_{b,i} + \beta^\star_{i,(b^\star-b)-1} ,
\end{equation}
which can be more clearly expressed by
\begin{equation}
    \q_{b+1} = 2\q_b + \boldsymbol{P}_{b+1} .
\end{equation}

Now we write the residual decomposition at precisions $b$ and $b+1$:
\begin{equation}
    \vw^\star = 2^{(b^\star-b)} \q_b + \delta_b = 2^{(b^\star-b)-1} \q_{b+1} + \delta_{b+1} .
\end{equation}
Substituting $\q_{b+1} = 2\q_b + \boldsymbol{P}_{b+1}$ yields
\begin{equation}
\begin{aligned}
    2^{(b^\star-b)} \q_b + \delta_b
    &= 2^{(b^\star-b)-1}(2\q_b + \boldsymbol{P}_{b+1}) + \delta_{b+1} \\
    \implies \quad
    \delta_b &= 2^{(b^\star-b)-1} \boldsymbol{P}_{b+1} + \delta_{b+1} .
\end{aligned}
\end{equation}
So, once $\q_b$ is fixed, the residual $\delta_b$ and the pair $(\boldsymbol{P}_{b+1},\delta_{b+1})$ describe each other up to an additive constant. Hence,
\begin{equation}
    \Ccal(\delta_b \mid \q_b) = \Ccal(\boldsymbol{P}_{b+1},\delta_{b+1} \mid \q_b) + O(1) .
\end{equation}

Applying the chain rule and absorbing the self-delimiting overhead into $O(\log(b^\star d))$, gives
\begin{equation}
\begin{aligned}
    \Ccal(\delta_b \mid \q_b)
    &= \Ccal(\boldsymbol{P}_{b+1} \mid \q_b) + \Ccal(\delta_{b+1} \mid \q_b,\boldsymbol{P}_{b+1}) + O(\log(b^\star d)) \\
    &= \Ccal(\boldsymbol{P}_{b+1} \mid \q_b) + \Ccal(\delta_{b+1} \mid \q_{b+1}) + O(\log(b^\star d)) .
\end{aligned}
\end{equation}
The last equality holds because $\q_{b+1}$ and $(\q_b,\boldsymbol{P}_{b+1})$ describe each other by a fixed rule, namely
\begin{equation}
    \q_{b+1}=2\q_b+\boldsymbol{P}_{b+1}, \qquad \q_b=\floor{\frac{\q_{b+1}}{2}}, \qquad \boldsymbol{P}_{b+1}=\q_{b+1} \bmod 2 .
\end{equation}
So replacing one condition with the other changes the conditional complexity by at most $O(1)$.

Using $R_b = \Ccal(\delta_q \mid \q_b)$ and $R_{b+1} = \Ccal(\delta_{b+1} \mid \q_{b+1})$,
\begin{equation}
\begin{aligned}
    R_b &= \Ccal(\boldsymbol{P}_{b+1} \mid \q_b) + R_{b+1} + O(\log(b^\star d)) \\
    \implies \quad
    R_b - R_{b+1} &= \Ccal(\boldsymbol{P}_{b+1} \mid \q_b) + O(\log(b^\star d)) .
\end{aligned}
\end{equation}
Since $\Ccal(\boldsymbol{P}_{b+1} \mid \q_b) \ge 0$, we get
\begin{equation}
    R_b \ge R_{b+1} - O(\log(b^\star d)) .
\end{equation}
Thus, the residual loss is non-increasing up to the same logarithmic term. If the new bit-plane is not already described by $\q_b$ and contributes information beyond that term, then $R_b > R_{b+1}$. \hfill \qedsymbol

\subsubsection{Cumulative Bit-Plane Refinement}
\begin{corollary}[Cumulative Bit-Plane Refinement]
\label{cor:cumulative_refinement}
For $1 \le b < r \le b^\star$, let $R_s := \Ccal(\delta_s \mid \q_s)$ and $s \in \{b,r\}$. Then we obtain
\begin{equation}
\begin{aligned}
    R_b &= \Ccal(\q_r \mid \q_b) + R_r + O(\log(b^\star d)) \\
    \implies \quad
    \Ccal(\q_r) - \Ccal(\q_b) &= R_b - R_r + O(\log(b^\star d)).
\end{aligned}
\end{equation}
In particular, setting $r=b^\star$ gives
\begin{equation}
    R_b = \Ccal(\vw^\star \mid \q_b) + O(\log(b^\star d)),
    \qquad
    R_{b^\star}=O(1),
\end{equation}
since $\q_{b^\star}=\vw^\star$ and $\delta_{b^\star}=0$. So adding bit-planes from precision $b$ to precision $r$ removes the information needed to refine $\q_b$ into $\q_r$, up to the same logarithmic term. At full precision, no residual information remains beyond the additive constant.
\end{corollary}

\begin{proof}\label{proof:cumulative_refinement}
    From Prop.~\ref{thm:res_decomp}, applied at precisions $b$ and $r$,
\begin{equation}
    \Ccal(\vw^\star) = \Ccal(\q_b) + R_b + O(\log(b^\star d)),
    \qquad
    \Ccal(\vw^\star) = \Ccal(\q_r) + R_r + O(\log(b^\star d)) .
\end{equation}
Subtracting gives
\begin{equation}
    \Ccal(\q_r) - \Ccal(\q_b) = R_b - R_r + O(\log(b^\star d)) .
\end{equation}

Now $\q_b$ is a fixed truncation of $\q_r$, so $\Ccal(\q_b \mid \q_r)=O(1)$, and hence
\begin{equation}
    \Ccal(\q_b,\q_r) = \Ccal(\q_r) + O(1) .
\end{equation}
Applying the chain rule to $(\q_b,\q_r)$ therefore gives
\begin{equation}
\begin{aligned}
    \Ccal(\q_r)
    &= \Ccal(\q_b) + \Ccal(\q_r \mid \q_b) + O(\log(b^\star d)) \\
    \implies \quad
    \Ccal(\q_r) - \Ccal(\q_b)
    &= \Ccal(\q_r \mid \q_b) + O(\log(b^\star d)) .
\end{aligned}
\end{equation}
Comparing the two expressions for $\Ccal(\q_r)-\Ccal(\q_b)$ yields
\begin{equation}
\begin{aligned}
    R_b - R_r
    &= \Ccal(\q_r \mid \q_b) + O(\log(b^\star d)) \\
    \implies \quad
    R_b
    &= \Ccal(\q_r \mid \q_b) + R_r + O(\log(b^\star d)) .
\end{aligned}
\end{equation}

Finally, set $r=b^\star$. Then $\q_{b^\star}=\vw^\star$ and $\delta_{b^\star}=0$, so
\begin{equation}
\begin{aligned}
    R_{b^\star} &= \Ccal(\delta_{b^\star} \mid \q_{b^\star}) = O(1) \\
    \implies \quad
    R_b &= \Ccal(\vw^\star \mid \q_b) + O(\log(b^\star d)) .
\end{aligned}
\end{equation}
Therefore, the residual loss at precision $b$ is the information needed to refine $\q_b$ into the target object $\vw^\star$, up to the same logarithmic term. As more bit-planes are retained, this missing information decreases; at full precision only the additive constant remains.

Repeated application of Prop.~\ref{thm:add_bitplane} gives the equivalent telescoping sum
\begin{equation}
    R_b - R_r = \sum_{j=b}^{r-1} \Ccal(\boldsymbol{P}_{j+1}\mid \q_j) + O((r-b)\log(b^\star d)) ,
\end{equation}
which shows that one conditional contribution is removed at each added bit-plane.
\end{proof}

\subsection{Empirical Validation of the Decrease in Residual Loss}\label{app:residual_validation}

\begin{figure}[b]
    \centering
    \includegraphics[width=1\linewidth]{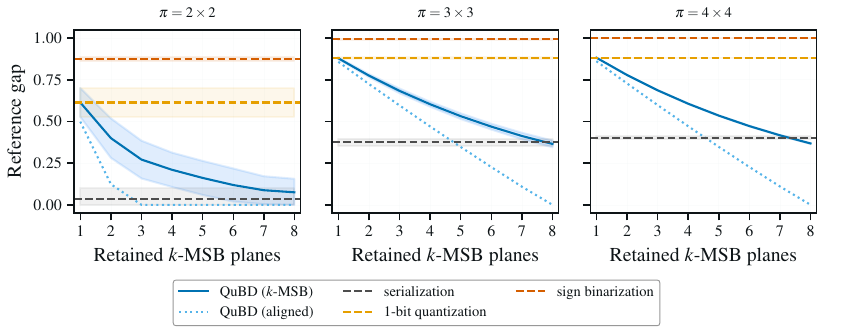}
    \caption{
    Relative gap ($\pm 1$ std.) to the finite CTM-support reference over $100$ trials (fixed symbol
    relabelings). The aligned QuBD control decreases as more MSB bit-planes are
    retained, validating the residual refinement predicted by Prop.~\ref{thm:add_bitplane}. Under symbol relabeling, QuBD remains closer to the reference than serialization, one-bit quantization, and sign binarization.}
    \label{fig:qubd_bitplane_loss}
\end{figure}

We test the estimation gap using the CTM table as a finite reference. By the coding theorem, CTM assigns low complexity to blocks with high empirical algorithmic probability, up to an additive
constant~\citep{solomonoff1964formal,levin1974laws,delahaye2012numerical}. For each block size $\pi$, we sum the CTM values of all supported blocks once. This gives a reference target object for the local structures available to the finite estimator, without BDM multiplicity terms.

We construct an $8$-bit object whose bit-planes expose all supported CTM blocks. For intuition, consider a binary CTM table with $1\times2$ blocks and support $\{00,01,10,11\}$. With $b=2$, choose $\boldsymbol{\beta}_1=0011$ and $\boldsymbol{\beta}_0=0101$, which gives the quantized object $[0,1,2,3]$. Under QuBD, $\Pi(\boldsymbol{\beta}_1)=\{00,11\}$ and $\Pi(\boldsymbol{\beta}_0)=\{01,01\}$, so the exposed support is $\{00,01,11\}$. The reverse construction, if the binary string is $00011011=00\,01\,10\,11$, then serialization exposes all four supported blocks.

For the aligned object, we note that the gap decreases as more bit-planes are retained, matching Prop.~\ref{thm:add_bitplane} and Cor.~\ref{cor:cumulative_refinement}. For the comparison, we keep the same target object and apply a random symbol relabeling before exposure. Each relabeling assigns a unique new value to each value in $\{0,\ldots,255\}$. This preserves the target object up to
an additive constant, while changing the binary blocks exposed to the CTM table. Fig.~\ref{fig:qubd_bitplane_loss} shows that QuBD remains closer to the finite reference than serialization, one-bit quantization, and sign binarization under these relabelings.

\section{Saturation in Finite-state Machines}
\subsection{Additional Saturation Details}
\label{app:saturation_details}

As discussed in Sec.~\ref{paragraph:exposure}, lossless representations can
still expose objects differently. Besides low residual loss, a representation
should expose substructures that remain informative under the finite
support of the CTM table, i.e., the set of structures available to the
estimator.

Fix a block size $\pi$ and a finite CTM table with supported states $S_\pi$.
Let $f(u)$ denote the empirical frequency of a block $u\in S_\pi$, similar to
Eq.~\eqref{eq:app_ctm_dist}. We write $z_{b,r}=\phi(\q_b)$ to denote the
exposure of $\vw$ with $b$-bit quantization and representation $r$. Partition
$z_{b,r}$ into blocks by $\Pi_\pi(z_{b,r})=(\w_1,\ldots,\w_m)$, where
$m=|z_{b,r}|/\pi$. For each supported block $u\in S_\pi$, define
\begin{equation}
\begin{aligned}
    c_\pi(u;z_{b,r})
    &:= |\{i\le m:\w_i=u\}|, \\
    A_\pi(z_{b,r})
    &:= \{u\in S_\pi:c_\pi(u;z_{b,r})>0\}, \\
    a_\pi(z_{b,r})
    &:= |A_\pi(z_{b,r})|.
\end{aligned}
\end{equation}
Here $c_\pi(u;z_{b,r})$ counts how often $u$ is exposed under $z_{b,r}$, while
$a_\pi(z_{b,r})$ is the number of supported blocks that appear. The KCS estimate
of Eq.~\eqref{eq:kc_bdm} can then be expressed as
\begin{equation}\label{eq:app_bdm_saturation_split}
    \Ccal_\BDM(z_{b,r})
    =
    \underbrace{
    \sum_{u\in A_\pi(z_{b,r})} -\log_2 f(u)
    }_{\text{new blocks}}
    +
    \underbrace{
    \sum_{u\in A_\pi(z_{b,r})} \log_2 c_\pi(u;z_{b,r})
    }_{\text{multiplicity}} .
\end{equation}
Importantly, after most supported blocks have appeared, the first term changes
little, and the estimate grows mainly through multiplicities. In this case, we
consider CTM tables in the finite setting saturated.

To compare exposures, let $S_{\pi,r}^{(b)}\subseteq S_\pi$ be the blocks that
can be exposed by $z_{b,r}$, and let
$H_{\pi,r}^{(b)}:=|S_{\pi,r}^{(b)}|$. A CTM table may support many structures,
but a given exposure may only reach some of them. Let $\rho_{\pi,r}(u)$ denote
the probability that one exposed block of $z_{b,r}$ equals $u$. Then, for $m$
exposed blocks, we have
\begin{equation}
\begin{aligned}
    \E[a_\pi(z_{b,r})]
    =
    \sum_{u\in S_{\pi,r}^{(b)}}
    \left[1-(1-\rho_{\pi,r}(u))^m\right],
    \quad
    \sigma_{\pi,r}(m)
    :=
    \frac{\E[a_\pi(z_{b,r})]}{H_{\pi,r}^{(b)}} .
\end{aligned}
\end{equation}
Here $\sigma_{\pi,r}(m)$ is the fraction of reachable support expected to appear
after $m$ exposed blocks. We call the estimator $(1-\epsilon)$-saturated when
$\sigma_{\pi,r}(m)\ge 1-\epsilon$.

\subsection{Proof of Proposition~\ref{prop:finite_saturation} (Finite-table Saturation)}
\label{proof:finite_saturation}

Fix $b$, $r$, $\pi$, and $S_{\pi,r}^{(b)}$. We omit the fixed $b$ for
readability and write $S_{\pi,r}$ with $H_{\pi,r}:=|S_{\pi,r}|$. Let
$\Pi_\pi(z_{b,r})=(\w_1,\ldots,\w_m)$ be the exposed blocks of size $\pi$.

For each $u\in S_{\pi,r}$, define
$I_u:=\mathbf{1}\{c_\pi(u;z_{b,r})>0\}$. Then
$a_\pi(z_{b,r})=\sum_{u\in S_{\pi,r}} I_u$, and by linearity
\begin{equation}
\begin{aligned}
    \E[a_\pi(z_{b,r})]
    =
    \E\left[\sum_{u\in S_{\pi,r}} I_u\right]
    =
    \sum_{u\in S_{\pi,r}}\E[I_u] 
    =
    \sum_{u\in S_{\pi,r}}
    \Prob(c_\pi(u;z_{b,r})>0).
\end{aligned}
\end{equation}
Let $\rho_{\pi,r}(u)$ be the probability that one exposed block equals $u$.
Since the exposed blocks are sampled independently, the probability that none
of the $m$ exposed blocks equals $u$ is $(1-\rho_{\pi,r}(u))^m$. Hence
\begin{equation}
    \E[a_\pi(z_{b,r})]
    =
    \sum_{u\in S_{\pi,r}}
    \left[1-(1-\rho_{\pi,r}(u))^m\right].
\end{equation}

In the uniform setting, which serves as a worst-case upper bound on the number of blocks required for saturation, we assume

$\rho_{\pi,r}(u)=1/H_{\pi,r}$ for all
$u\in S_{\pi,r}$, and therefore
\begin{equation}
\begin{aligned}
    \E[a_\pi(z_{b,r})]
    &=
    \sum_{u\in S_{\pi,r}}
    \left[1-\left(1-\frac{1}{H_{\pi,r}}\right)^m\right] \\
    &=
    H_{\pi,r}
    \left[1-\left(1-\frac{1}{H_{\pi,r}}\right)^m\right]
    \approx
    H_{\pi,r}\left(1-e^{-m/H_{\pi,r}}\right),
\end{aligned}
\end{equation}
where the last step uses $(1-x)^m\approx e^{-mx}$ for small $x$. By definition,
$\sigma_{\pi,r}(m)= \frac{\E[a_\pi(z_{b,r})]}{H_{\pi,r}}$, so
\begin{equation}
\begin{aligned}
    \sigma_{\pi,r}(m)
    &=
    \frac{
    H_{\pi,r}\left[1-\left(1-\frac{1}{H_{\pi,r}}\right)^m\right]
    }{H_{\pi,r}} \\
    &=
    1-\left(1-\frac{1}{H_{\pi,r}}\right)^m
    \approx
    1-e^{-m/H_{\pi,r}} .
\end{aligned}
\end{equation}
Thus, saturation is controlled by $m/H_{\pi,r}$. We now write $a_\pi(m)$ for the number of distinct supported blocks seen for an exposure of $m$ blocks. Then the expected number of new supported blocks added by one more exposed block is
\begin{equation}
\begin{aligned}
    \E[a_\pi(m+1)]-\E[a_\pi(m)]
    &=
    H_{\pi,r}
    \left[
    \left(1-\frac{1}{H_{\pi,r}}\right)^m
    -
    \left(1-\frac{1}{H_{\pi,r}}\right)^{m+1}
    \right] \\
    &=
    \left(1-\frac{1}{H_{\pi,r}}\right)^m
    \approx
    e^{-m/H_{\pi,r}} .
\end{aligned}
\end{equation}
Thus, after $m$ blocks pass the saturation threshold, additional exposed blocks
mostly repeat blocks already present in the finite table, and the estimate grows
through multiplicity.

Finally, $(1-\epsilon)$-saturation means
$\sigma_{\pi,r}(m)\ge 1-\epsilon$, so
\begin{equation}
\begin{aligned}
    \sigma_{\pi,r}(m)\ge 1-\epsilon
    &\Longleftrightarrow
    1-\left(1-\frac{1}{H_{\pi,r}}\right)^m \ge 1-\epsilon \\
    &\Longleftrightarrow
    \left(1-\frac{1}{H_{\pi,r}}\right)^m \le \epsilon \\
    &\Longleftrightarrow
    m \log\!\left(1-\frac{1}{H_{\pi,r}}\right) \le \log \epsilon \\
    &\Longleftrightarrow
    m \ge
    \frac{\log \epsilon}{\log(1-1/H_{\pi,r})}
    \approx
    H_{\pi,r}\log\frac{1}{\epsilon},
\end{aligned}
\end{equation}
where the inequality flips in the last step since
$\log(1-1/H_{\pi,r})<0$, and the approximation uses
$\log(1-x)\approx -x$.

After saturation, the finite table mostly contributes through repeated counts.
For uniform exposure, the expected count of each supported block is
$m/H_{\pi,r}$, and for large $m$ these counts concentrate around this value.
And when $a_\pi(z_{b,r})\approx H_{\pi,r}$,
\begin{equation}
\begin{aligned}
    \sum_{u\in A_\pi(z_{b,r})}\log_2 c_\pi(u;z_{b,r})
    \approx
    \sum_{u\in S_{\pi,r}}
    \log_2\!\left(\frac{m}{H_{\pi,r}}\right) 
    =
    H_{\pi,r}\log_2\!\left(\frac{m}{H_{\pi,r}}\right).
\end{aligned}
\end{equation}
Thus, after saturation, further growth in the estimate comes mainly from
multiplicity terms rather than newly observed blocks. \hfill $\square$

\subsection{Saturation Across Exposures}\label{app:saturation_exposure}
We compare binarization, serialization, and bit-plane decomposed exposures by counting how many binary blocks each representation exposes in the CTM table. Let $d$ be the number of quantized symbols in $\q_b\in \Sigma_b^d$, and let $\pi$ be the binary block size. Since the available CTM tables are two-symbol tables, every exposure is evaluated through binary blocks in $S_\pi\subseteq\{0,1\}^{\pi}$. The relevant quantity is the number of exposed blocks $m_r$ relative to the reachable support $H_{\pi,r}^{(b)}$, and not only $d$, the number of $b$-ary symbols alone. 

For the three exposures, the number of blocks is
\begin{equation}
    m_{\mathrm{bin}}=\frac{d}{\pi},\qquad
    m_{\mathrm{ser}}=\frac{bd}{\pi},\qquad
    m_{\mathrm{plane},j}=\frac{d}{\pi},\quad j=1,\ldots,k ,
\end{equation}
where $m_{\mathrm{plane},j}$ is counted per retained bit-plane. If $k$ bit-planes
are retained, the estimator considers $k$ binary objects separately and independently, and the
aggregate reachable support is
\begin{equation}
    H_{\pi,\mathrm{plane}}^{(k)}
    =
    \sum_{j=1}^{k}H_{\pi,\mathrm{plane},j}
    =
    kH_{\pi,\mathrm{plane}} .
\end{equation}

By Prop.~\ref{prop:finite_saturation}, the object sizes needed before
$(1-\epsilon)$-saturation are
\begin{equation}
\begin{aligned}
    d_{\mathrm{bin}}
    &\gtrsim \pi H_{\pi,\mathrm{bin}}\log\frac{1}{\epsilon}, \qquad
    d_{\mathrm{ser}}
    \gtrsim \frac{\pi H_{\pi,\mathrm{ser}}^{(b)}}{b}\log\frac{1}{\epsilon}, \\
    d_{\mathrm{plane},j}
    &\gtrsim \pi H_{\pi,\mathrm{plane},j}\log\frac{1}{\epsilon},
    ~~~\text{and}~~~
    d_{\mathrm{plane}}^{(k)}
    \gtrsim
    k \pi H_{\pi,\mathrm{plane}}^{(k)}\log\frac{1}{\epsilon}.
\end{aligned}
\end{equation}

Binarization exposes only $d/\pi$ blocks, but estimates a one-bit target. Serialization preserves the
$b$-bit target, but exposes $bd/\pi$ blocks to one finite table and therefore
saturates earlier. Bit-plane exposure preserves the $b$-bit target and scores
the planes separately. We note that increasing $k$ reduces residual loss and increases the aggregate support available to the estimator, while each plane is still exposed at the slower rate per bit-plane.
For a concrete example, take $b=8$ and $\pi=16$, corresponding to binary $4\times4$
blocks. The CTM support is of size $|H|=2^{16}$. With $\epsilon=0.1$, we have
\begin{equation}
    H\log\frac{1}{\epsilon}
    =
    2^{16}\log\frac{1}{0.1}
    \approx
    1.51\times 10^5 .
\end{equation}
The corresponding object scales are
\begin{equation}
\begin{aligned}
    d_{\mathrm{ser}}
    &\gtrsim
    \frac{\pi H}{b}\log\frac{1}{\epsilon}
    =
    \frac{16}{8}\cdot 1.51\times 10^5
    \approx
    3.0\times 10^5, \\
    d_{\mathrm{bin}}
    &\gtrsim
    \pi H\log\frac{1}{\epsilon}
    =
    16\cdot 1.51\times 10^5
    \approx
    2.4\times 10^6, \\
    d_{\mathrm{plane}}
    &\gtrsim
    \pi H\log\frac{1}{\epsilon}
    =
    16\cdot 1.51\times 10^5
    \approx
    2.4\times 10^6 .
\end{aligned}
\end{equation}
Serialization reaches the same saturation about
$b=8$ times earlier than one bit-plane. Binarization has the same saturation
scale as a single bit-plane in this example, but it estimates a
different one-bit target and discards the lower-bit residuals. We include it merely for completeness. So bit-plane exposure keeps the $b$-bit target while delaying saturation in the number of bit-planes. The delayed saturation rates across the different exposures are empirically observed in Fig.~\ref{fig:saturation_limits}. Importantly, the estimation gap for increasing object sizes caused by the different representations, and when they fail, as discussed in Sec.~\ref{sec:saturation}, is seen in Figs.~\ref{fig:rep_failure} and~\ref{fig:rep_failure_100M}. This shows how the delayed saturation yields more reliable estimates for larger objects.  

\begin{figure}[htbp]
    \centering
    \includegraphics[width=1\linewidth]{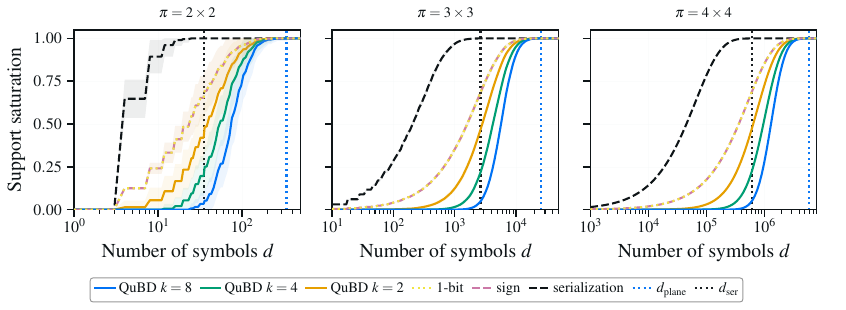}
    \caption{Empirical saturation rates as the number of symbols $d$ exposed to the CTM table increases across block sizes $\pi\in\{2\times2,3\times3,4\times4\}$. Each curve shows the fraction of available support observed for each representation. Vertical lines denote the theoretical saturation rates derived for bit-plane decomposition in QuBD ($d_{\mathrm{plane}}$) and serialized ($d_{\mathrm{ser}}$) representations. Symbols are sampled uniformly at random from the support of the CTM-table for up to $10^6$ draws. Results are averaged over 100 trials.}
    \label{fig:saturation_limits}
\end{figure}

\begin{figure}[htbp]
    \centering
    \includegraphics[width=\linewidth]{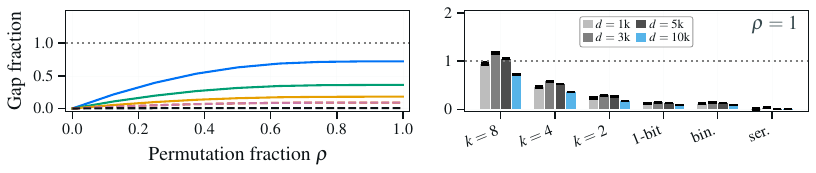}
    \includegraphics[width=\linewidth]{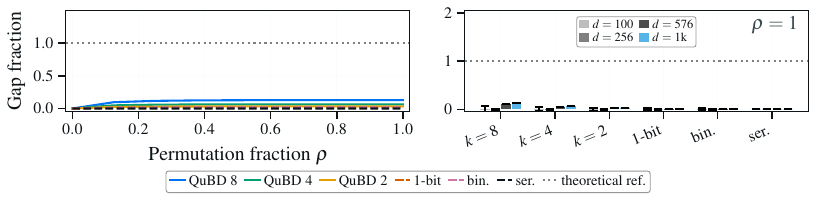}
    \caption{Complexity gap with \textbf{(top)} $\pi = (3\times 3)$ and \textbf{(bottom)} $\pi = (2\times 2)$, between estimations of an ordered $8$-ary object and its permutation, where $\rho$ is the fraction of symbols $d$ permuted. The gap is shown \textbf{(left)} for the largest object in consideration, $d \in \{1\text{k}, 10\text{k} \}$ symbols, as $\rho$ increases, and \textbf{(right)} across object sizes for different methods when $\rho =1$. QuBD keeps a larger and more stable estimation gap than serialization, one-bit, and sign-binarization. Complexity values are normalized by the theoretical reference for its type class~\citep{cover1999elements}. Each object is sampled uniformly at random, and results are averaged over 100 trials.}
    \label{fig:rep_failure}
\end{figure}

\begin{figure}
    \centering
    \includegraphics[width=\linewidth]{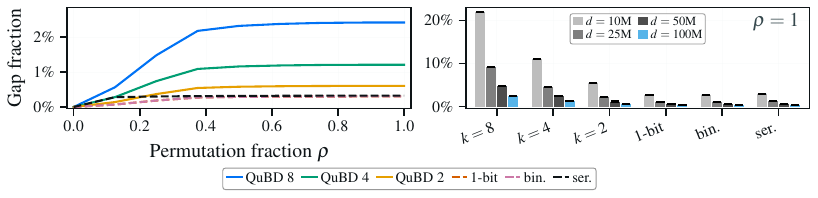}
    \caption{As Fig.~\ref{fig:rep_failure} for  $\pi = (4 \times 4)$ with normalized $(\%)$ gap shown (\textbf{left}) for an object with $d=100$M symbols and (\textbf{right)} across object sizes.}
    \label{fig:rep_failure_100M}
\end{figure}

\clearpage \newpage
\section{Additional Experimental Details}\label{app:experimental_details}

The accuracy-versus-complexity, learning curve, and grokking experiments were performed on a desktop workstation with an Intel i7 processor, 64 GB memory, and an Nvidia RTX 5090 GPU with 32 GB memory. 
The QuBD complexity across bit-planes, per-layer, and PTQ experiments were conducted on an Apple Silicon Mac, with the latter using the MPS backend. 
All code was implemented in PyTorch~\citep{paszke2019pytorch}. Specific choices for individual experiments are presented next. 
\looseness=-1

\subsection{Complexity Across Models and Bit-Planes Experiment}\label{app:timm_bit_planes}
The results reported in Fig.~\ref{fig:complexity} and Fig.~\ref{fig:complexity_full} are based on 100 pretrained models sourced from the \texttt{timm} library~\citep{rw2019timm}, spanning convolutional, transformer, and hybrid architectures. Each model is loaded with its default pretrained weights and \texttt{timm}'s default random initialization under a fixed seed. 
Following the weight tensor preprocessing described in Appx.~\ref{app:weight_preprocess}, each weight tensor is quantized to 8 bits with a uniform quantizer, clipping at the 99.9th percentile. This experiment requires no dataset evaluation and took approximately 80 minutes to complete on the reported hardware. 

The 100 models, ordered by number of parameters, are: \texttt{edgenext\_xx\_small} (1.3M), \texttt{regnetx\_002} (2.7M), \texttt{starnet\_s1} (2.9M), \texttt{rexnet\_100} (4.8M), \texttt{convnext\_femto} (5.2M), \texttt{efficientnet\_es\_pruned} (5.4M), \texttt{tf\_efficientnet\_es} (5.4M), \texttt{deit\_tiny\_distilled\_patch16\_224} (5.9M), \texttt{hgnetv2\_b0} (6.0M), \texttt{tf\_mixnet\_l} (7.3M), \texttt{nf\_regnet\_b1} (10.2M), \texttt{resnext26ts} (10.3M), \texttt{eca\_resnext26ts} (10.3M), \texttt{pit\_xs\_224} (10.6M), \texttt{mobilenetv4\_hybrid\_medium} (11.1M), \texttt{shvit\_s3} (14.2M), \texttt{vit\_pwee\_patch16\_reg1\_gap\_256} (15.3M), \texttt{convnext\_nano} (15.6M), \texttt{resnet26} (16.0M), \texttt{rexnet\_200} (16.4M), \texttt{cspresnext50} (20.6M), \texttt{dla60\_res2net} (20.8M), \texttt{gernet\_m} (21.1M), \texttt{resnet50\_clip\_gap} (23.5M), \texttt{efficientvit\_b2} (24.3M), \texttt{resnetaa50d} (25.6M), \texttt{resnet50d} (25.6M), \texttt{dla102x} (26.3M), \texttt{swin\_s3\_tiny\_224} (28.3M), \texttt{focalnet\_tiny\_lrf} (28.6M), \texttt{poolformerv2\_s36} (30.8M), \texttt{maxvit\_tiny\_tf\_384} (31.0M), \texttt{fasternet\_s} (31.2M), \texttt{selecsls42b} (32.5M), \texttt{sam2\_hiera\_small} (33.9M), \texttt{cs3darknet\_x} (35.0M), \texttt{fastvit\_mci2} (35.8M), \texttt{resnet61q} (36.8M), \texttt{levit\_conv\_384} (39.1M), \texttt{levit\_384} (39.1M), \texttt{eca\_nfnet\_l1} (41.4M), \texttt{coatnet\_rmlp\_1\_rw\_224} (41.7M), \texttt{xception71} (42.3M), \texttt{crossvit\_18\_dagger\_240} (44.3M), \texttt{crossvit\_18\_dagger\_408} (44.6M), \texttt{efficientnet\_h\_b5} (45.9M), \texttt{xcit\_small\_24\_p8\_384} (47.6M), \texttt{legacy\_seresnet101} (49.3M), \texttt{inception\_next\_small} (49.4M), \texttt{swinv2\_cr\_small\_ns\_224} (49.7M), \texttt{cs3se\_edgenet\_x} (50.7M), \texttt{gcvit\_small} (51.1M), \texttt{dla169} (53.4M), \texttt{tf\_efficientnetv2\_m} (54.1M), \texttt{regnetx\_160} (54.3M), \texttt{tresnet\_l} (56.0M), \texttt{convformer\_m36} (57.1M), \texttt{repvgg\_b1} (57.4M), \texttt{regnetz\_e8} (57.7M), \texttt{nextvit\_large} (57.9M), \texttt{resnet152c} (60.2M), \texttt{vit\_betwixt\_patch16\_reg1\_gap\_256} (60.4M), \texttt{twins\_pcpvt\_large} (61.0M), \texttt{pvt\_v2\_b4} (62.6M), \texttt{resnetrs101} (63.6M), \texttt{resnet200d} (64.7M), \texttt{seresnet152d} (66.8M), \texttt{cait\_s36\_384} (68.4M), \texttt{wide\_resnet50\_2} (68.9M), \texttt{maxvit\_small\_tf\_384} (69.0M), \texttt{swin\_s3\_base\_224} (71.1M), \texttt{dm\_nfnet\_f0} (71.5M), \texttt{hgnet\_base} (71.6M), \texttt{poolformer\_m48} (73.5M), \texttt{pit\_b\_224} (73.8M), \texttt{coatnet\_2\_rw\_224} (73.9M), \texttt{coatnet\_rmlp\_2\_rw\_224} (73.9M), \texttt{pit\_b\_distilled\_224} (74.8M), \texttt{hrnet\_w48} (77.5M), \texttt{tresnet\_xl} (78.4M), \texttt{resnext101\_64x4d} (83.5M), \texttt{repvgg\_b3g4} (83.8M), \texttt{xcit\_medium\_24\_p16\_384} (84.4M), \texttt{xcit\_medium\_24\_p16\_224} (84.4M), \texttt{vit\_base\_patch16\_dinov3} (85.6M), \texttt{deit\_base\_patch16\_224} (86.6M), \texttt{vit\_base\_patch14\_reg4\_dinov2} (86.6M), \texttt{inception\_next\_base} (86.7M), \texttt{rdnet\_base} (87.5M), \texttt{swin\_base\_patch4\_window7\_224} (87.8M), \texttt{davit\_base} (88.0M), \texttt{focalnet\_base\_srf} (88.1M), \texttt{vit\_base\_patch32\_224} (88.2M), \texttt{resnext101\_32x8d} (88.8M), \texttt{gcvit\_base} (90.3M), \texttt{fasternet\_l} (93.5M), \texttt{seresnext101\_32x8d} (93.6M), \texttt{seresnext101d\_32x8d} (93.6M), \texttt{caformer\_b36} (98.8M), and \texttt{convformer\_b36} (99.9M).
  
\clearpage
\subsection{Accuracy-versus-Complexity Experiment}
\label{app:pareto}

The results reported in Fig.~\ref{fig:data-budget} are based on MLPs trained on Fashion-MNIST~\citep{xiao2017fashion}. Four-layered MLPs of varying capacities were configured by doubling the number of hidden units.  The base model had $[1024,512,256]$ hidden units; the three other versions were obtained by scaling the hidden units by the factor $s=[0.5,2.0,4.0]$, resulting in the four MLPs reported. All the models were trained for 100 epochs using 40k training samples, tested on 10k test samples using the Adam optimizer~\citep{kingma2015adam} with a learning rate of $3\times 10^{-4}$ and weight decay of $0.01$ with a batch size of $128$. All experiments were repeated over three random seeds, and we report the mean and standard deviation of all the metrics. Each experimental run to generate the data in Fig.~\ref{fig:data-budget} took about 4.5 hours on the hardware described, and through the project lifetime these experiments were conducted no less than 10 times, implying an overall compute budget of about 50 hours. 

\subsection{Learning Curve and Grokking Experiment}
\label{app:learning_curve}

Learning curves reported in Fig.~\ref{fig:training} follow a similar setup as reported above. Tiny-ViT~\citep{tinyvit} on CIFAR-10~\citep{krizhevsky2009learning} was trained for 300 epochs. Each Tiny-ViT experiment took about 6 hours, and these experiments were repeated over three random seeds.

For the Grokking experiment in~\cite{power2022grokking}, we use a 2-layered MLP with $[512, 256]$ hidden units to model the modular addition task $z = (x + y) \pmod{P}$ with $P=97$. The dataset, consisting of the full $P^2$ Cartesian product, was partitioned into a 50\% training split and a 50\% validation split. The model was trained for 11k epochs, and metrics, including QuBD complexity, were logged during training/validation. The grokking experiments took about 5 minutes on the reported hardware and were repeated over three random seeds.

\subsection{Continued Simplification at LLM Scales}\label{app:pythia}

We further study whether the relation between training dynamics and QuBD complexity extends to frontier AI models, using the Pythia large language model (LLM) family, which was trained using The Pile dataset comprising about 300B tokens~\citep{biderman2023pythia}. We use the training loss and model checkpoints provided by the authors to estimate the QuBD complexity along the training trajectory across four parameter scales: 70M, 160M, 410M, and 1B parameters, illustrated in~Fig.~\ref{fig:pythia}. Across all scales, we observe that QuBD complexity of the MSB is nearly flat and decreases later when the loss plateaus, matching the grokking experiment in Fig.~\ref{fig:training} (\textbf{right}). 

Furthermore, we note that binarized BDM complexity fails to track the complexity reduction, whereas QuBD, even with only the MSB-plane better tracks the reduction, consistent with the analysis in Sec.~\ref{sec:qubd_theory} and our findings in Fig.~\ref{fig:training} (\textbf{left} and \textbf{center}). This exemplifies, additionally, that QuBD complexity does not coincide with sign-binarization, as discussed in Sec.~\ref{sec:quant_influence}.

\begin{figure}[htbp]
    \centering
    \includegraphics[width=0.44\linewidth]{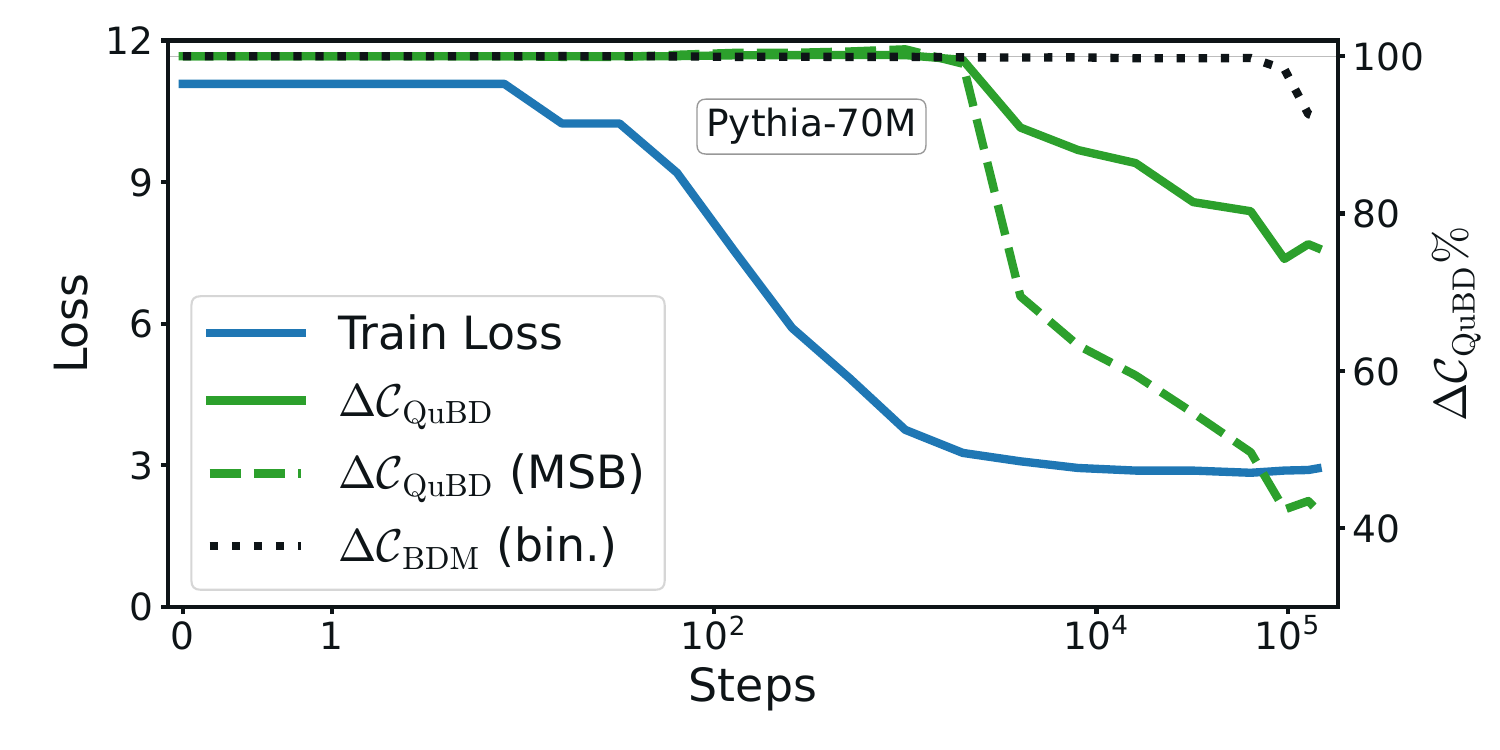}
    \includegraphics[width=0.44\linewidth]{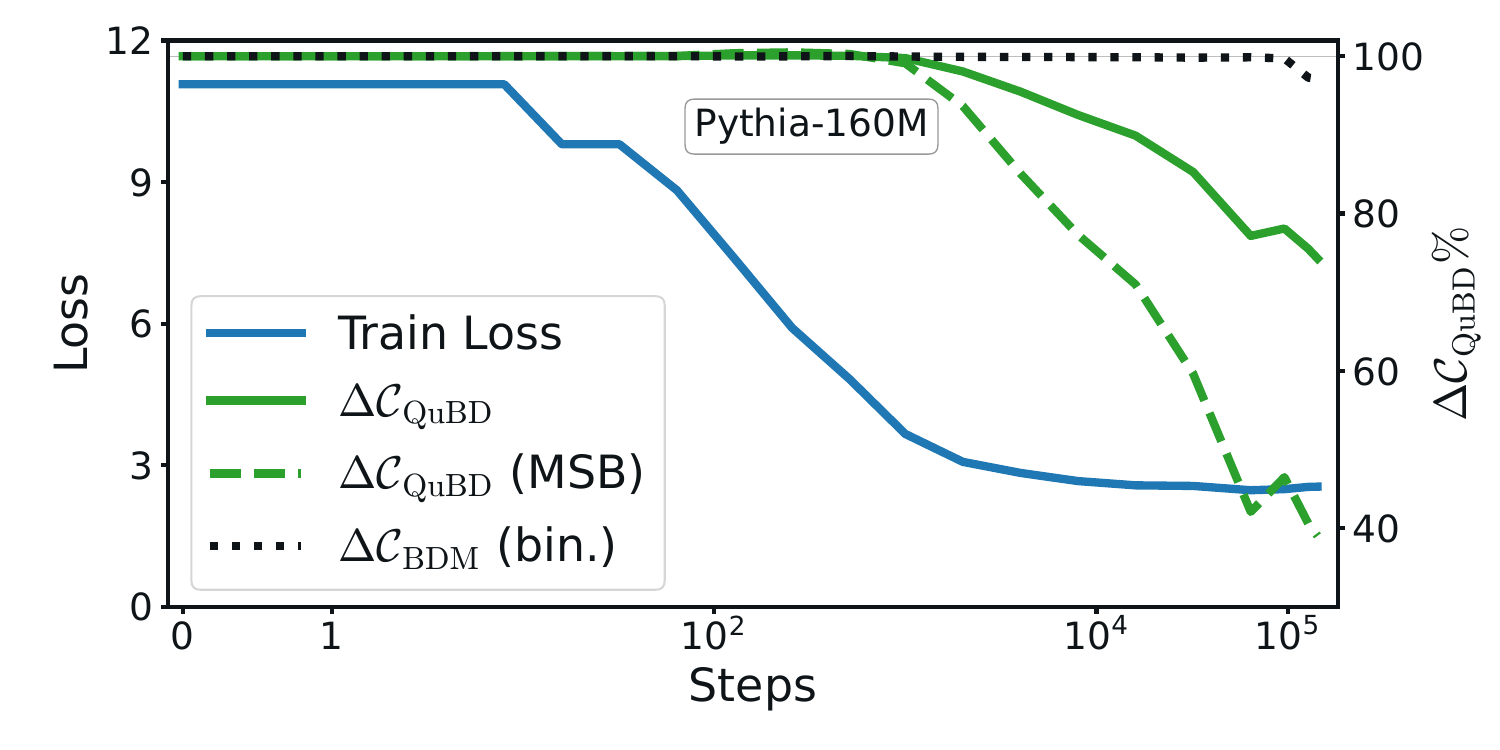}
    \includegraphics[width=0.44\linewidth]{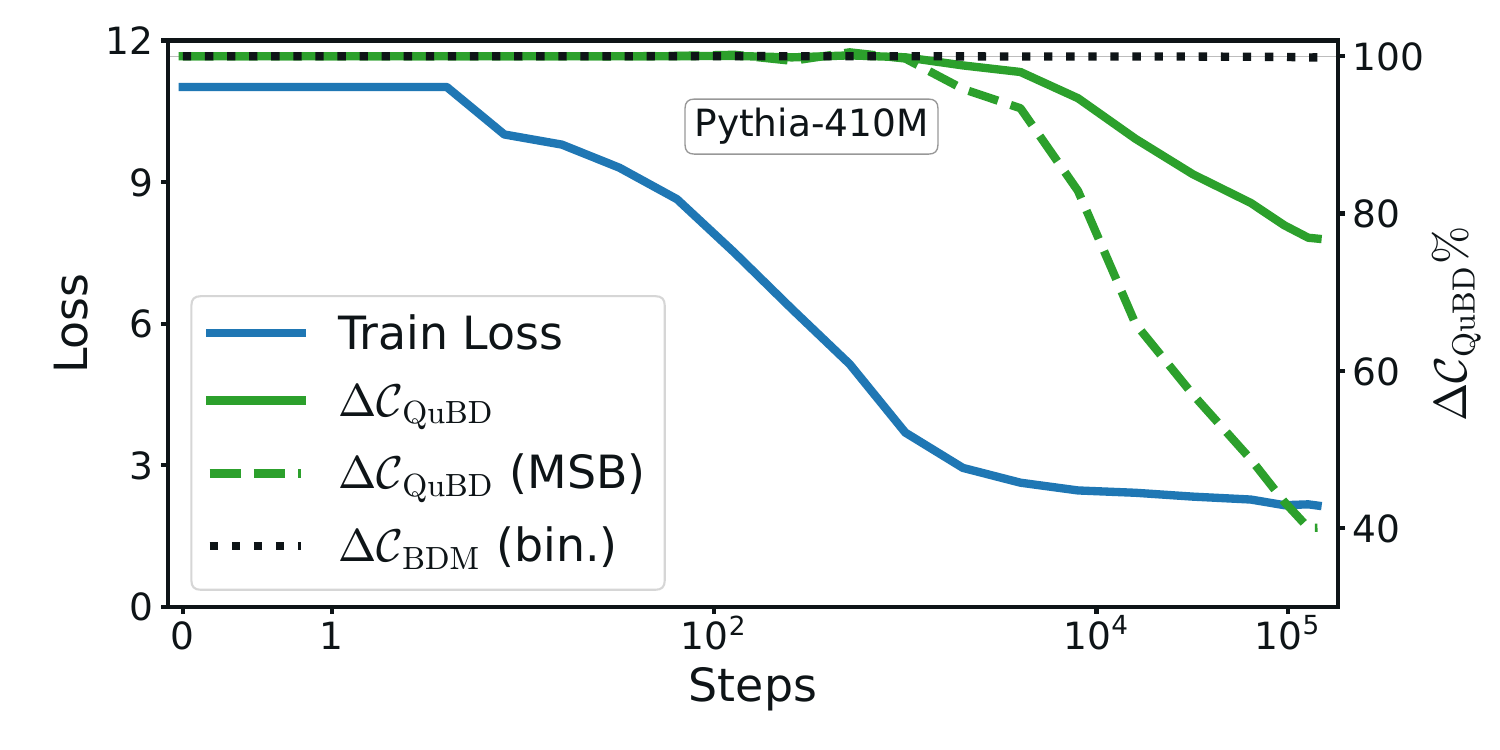}
    \includegraphics[width=0.44\linewidth]{figures/pythia/pythia-1b_dynamics.pdf}
    \caption{Learning curves along $\Delta \Ccal_{\text{QuBD}}$, an MSB-plane only variant, and $\Delta \Ccal_{\text{BDM}}$ evolution for increasing sizes of the Pythia model family with 70M, 160M, 410M, and 1B parameters.
    } 
    \label{fig:pythia}
\end{figure}

\subsection{Complexity per Layer Experiment}\label{app:per_layer}
The results reported in Fig.~\ref{fig:qubd_reduction_per_layer_resnet18} and Fig.~\ref{fig:qubd_reduction_per_layer_resnet50} are based on pretrained ResNet-18 and ResNet-50 models sourced from the \texttt{timm} library~\citep{rw2019timm}, compared against the same randomly initialized counterparts as in Appx~\ref{app:timm_bit_planes}. For each layer, the weights are preprocessed as described in Appx.~\ref{app:weight_preprocess} and quantized to 8 bits, clipping at the 99.9th percentile. The QuBD complexity ratio $\Delta \Ccal_\text{QuBD}$ is then computed per layer rather than summed over the model, and is reported for the two most significant bit-planes (P7, P6). The experiment requires no dataset evaluation and took about 2 minutes to complete. 

\subsection{Post-Training Quantization Experiment}\label{app:ptq}
The results reported in Fig. \ref{fig:qubd_ptq} are based on five pretrained models sourced from the \texttt{timm} library~\citep{rw2019timm}: ResNet-18, ResNet-50, ViT-B/16, EfficientNet-B0, and MobileNetV3-Large, all trained on ImageNet-1k~\citep{deng2009imagenet}. 
The full-precision (FP32) weights are quantized to 16 bits, clipping at the 99.9th percentile, and compared against the same randomly initialized counterparts as in Appx.~\ref{app:timm_bit_planes}, following the weight tensor preprocessing described in Appx.~\ref{app:weight_preprocess}
QuBD complexity is reported for all 16 bit-planes. 
For post-training quantization (PTQ), the weights of all linear and convolutional layers are quantized per output channel, without any fine-tuning or calibration data. For 2 to 8 bits, we use a symmetric uniform quantizer with $2^b - 1$ levels whose scale is set by the channel's largest absolute weight. For 1 bit, each weight is replaced by its sign, scaled by the channel's mean absolute weight. Biases, normalization parameters, and activations are kept in FP32. All models were evaluated on the full ImageNet-1k validation set (50,000 samples). The FP32 baselines were evaluated the same way. The total compute time for this experiment was approximately 3.5 hours on the reported hardware.

\section{Additional Results}

\subsection{QuBD Complexity Ratio Across Eight Bit-Planes and Clipping Percentiles}\label{app:qubd_ratio}
Fig. \ref{fig:complexity_full} shows the QuBD complexity ratio across all eight bit-planes for 100 pretrained models from the {\tt timm} library~(\citet{rw2019timm}, see Appx.~\ref{app:timm_bit_planes}), clipping at the 99.9th percentile as in the main text.
Figs.~\ref{fig:complexity_full-9999} and~\ref{fig:complexity_full-noclip} show the same models when clipping at the 99.99th percentile and without clipping. The complexity ratio generally decreases as clipping is relaxed, as extreme weights stretch the quantization range. 
For the most significant bit-planes, the median is 89\% at the 99.9th percentile, 78\% at the 99.99th percentile, and 35\% without clipping.

\begin{figure}[htbp]
    \centering
    \includegraphics[width=0.9\linewidth]{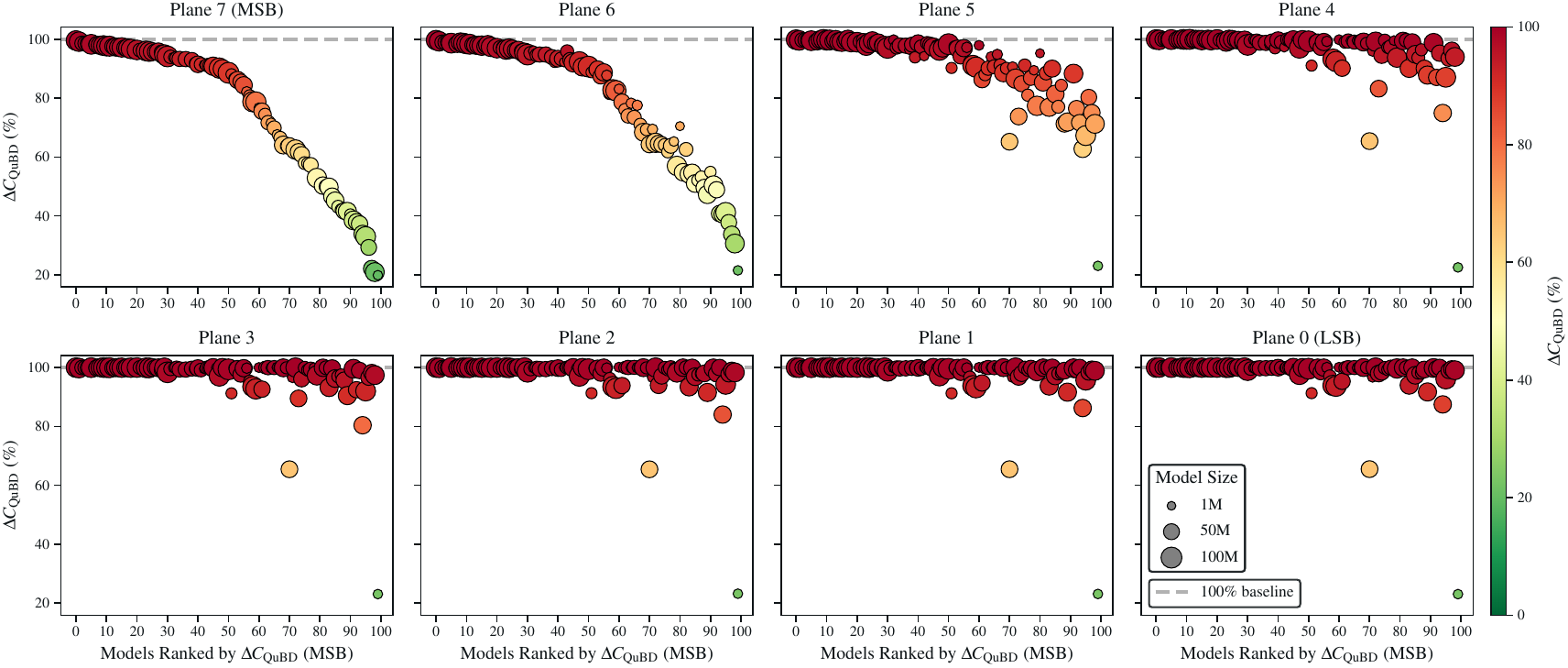}
    \caption{As Fig. \ref{fig:complexity} but shown for all eight bit-planes, clipping at the 99.9th percentile.
    }
    \label{fig:complexity_full}
\end{figure}

\begin{figure}[htbp]
    \centering
    \includegraphics[width=0.9\linewidth]{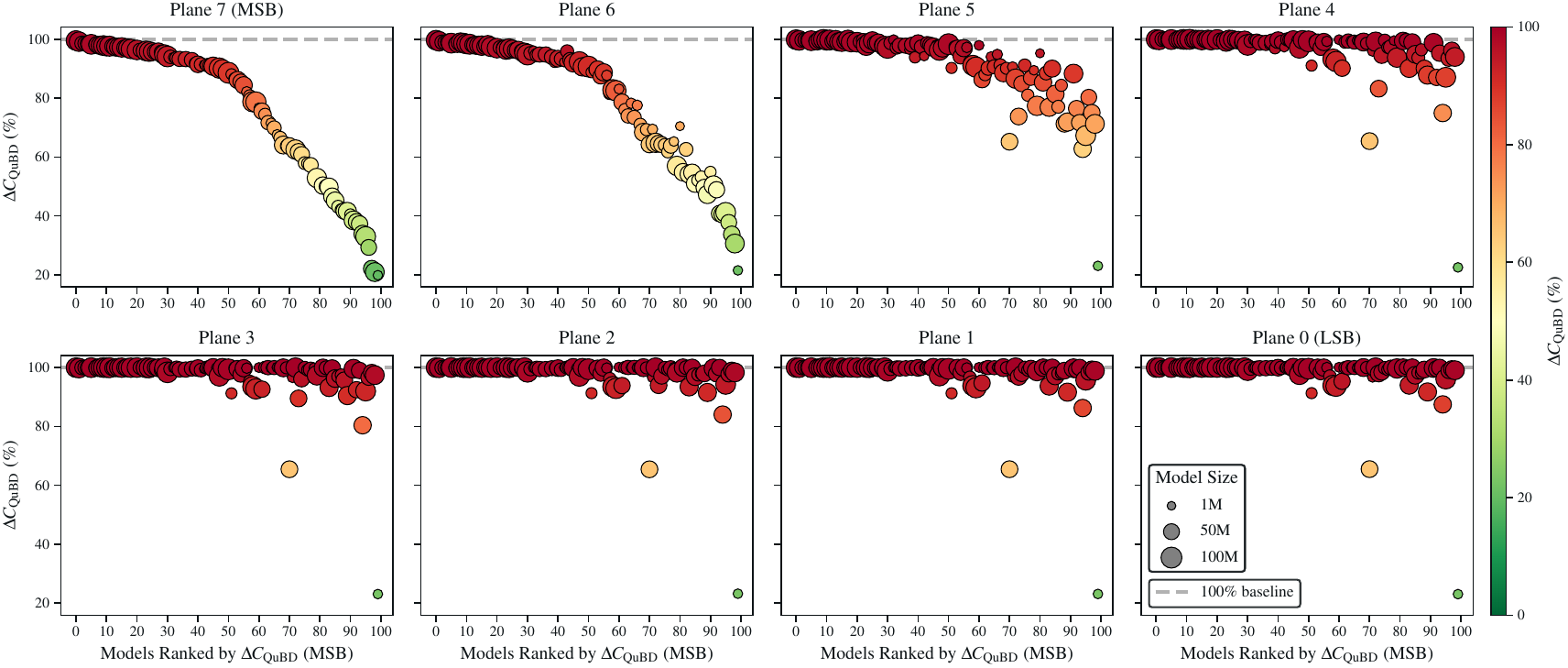}
    \caption{As Fig.~\ref{fig:complexity_full}, clipping at the 99.99th percentile. Models are ranked by their complexity ratio on the most significant bit-plane (Plane 7, MSB) for this setting.}
    \label{fig:complexity_full-9999}
\end{figure}

\begin{figure}[htbp]
    \centering
    \includegraphics[width=0.9\linewidth]{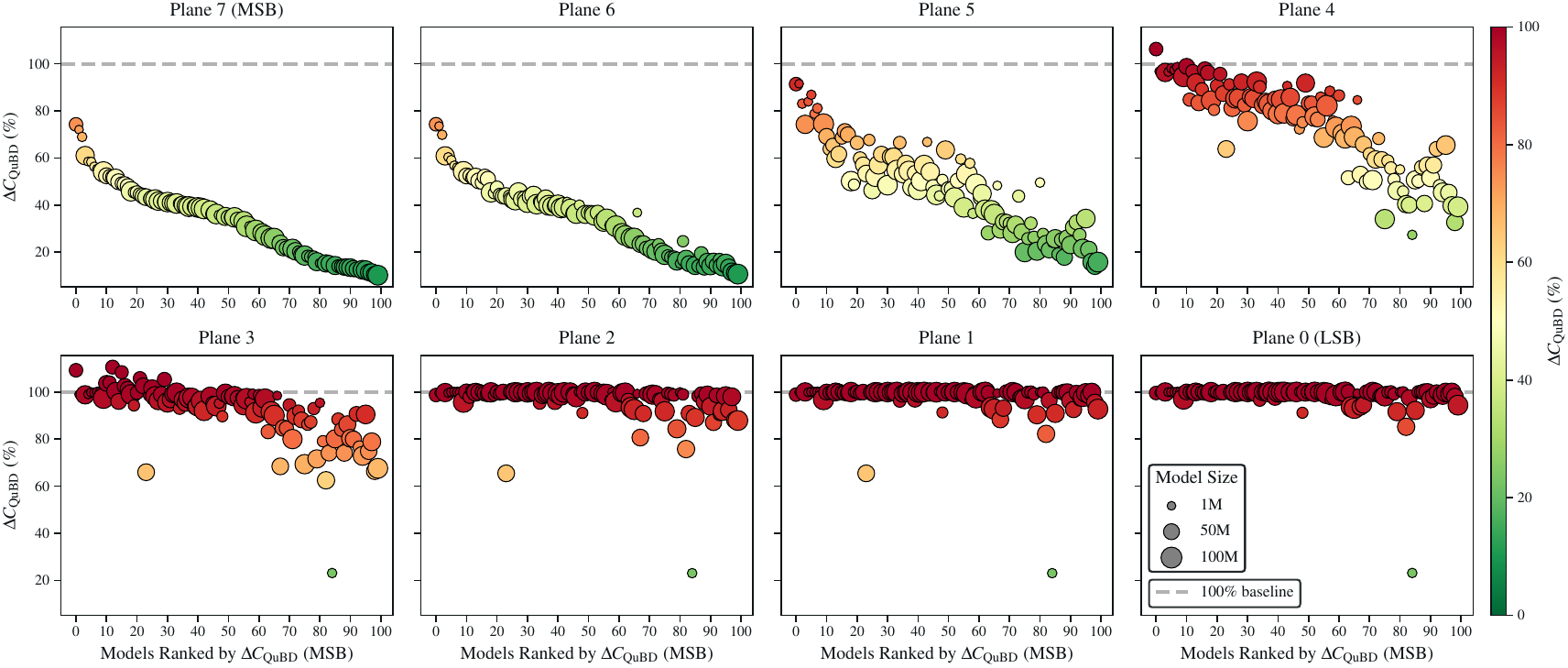}
    \caption{As Fig.~\ref{fig:complexity_full}, without clipping. Models are ranked by their complexity ratio on the most significant bit-plane (Plane 7, MSB) for this setting.}
    \label{fig:complexity_full-noclip}
\end{figure}

\vspace{2cm}
\subsection{Robustness to Percentile-clipping on Complexity Reduction.}\label{app:percentile_bdm_vs_qubd}
In relation to the quantized bit-plane decomposition, we note that the MSB would coincide when $\w_{\min} \approx \w_{\max}$, resulting in $\w < 0 \to 0$ and $\w > 0 \to 1$, similar to sign-binarization. One concern is then that the QuBD complexity is impacted by the weight distribution caused by percentile-clipping. For instance, a single outlier can shift the midpoint and flip the MSB for weights with smaller magnitude, giving rise to disparate structures. 
\begin{wraptable}{r}{0.45\textwidth}
\vspace{-0.5em}
    \centering
    \small
    \caption{Complexity reduction with 99.9 percentile-clipping
    over three random seeds.}
    \label{tab:clipped_symmetric_quantization}
    \begin{tabular}{lrr}
        \toprule
        No. params. & $\Delta\mathcal{C}_{\mathrm{BDM}}$ (\%)
                    & $\Delta\mathcal{C}_{\mathrm{QuBD}}$ (\%) \\
        \midrule
        0.57M  & $40.84 \pm 0.89$ & $89.59 \pm 0.48$ \\
        1.46M  & $34.28 \pm 0.67$ & $94.88 \pm 0.04$ \\
        4.24M  & $25.90 \pm 0.43$ & $96.99 \pm 0.10$ \\
        13.71M & $21.21 \pm 0.25$ & $96.28 \pm 0.43$ \\
        \bottomrule
    \end{tabular}
\end{wraptable}
In Tab.~\ref{tab:clipped_symmetric_quantization} we compare to a symmetric range quantizer using 99.9th percentile-clipping at the final sample budget and see that the trends observed in Fig.~\ref{fig:data-budget} persist across all models. As both methods use similar CTM tables and reshaping, the remaining difference stems from the additional bit-planes retained by QuBD, which is consistent with the quantizer ablation in Sec.~\ref{sec:quant_influence} (Tab.~\ref{tab:qat}) and with the estimation-gap analysis in Sec.~\ref{sec:qubd_theory}. Moreover, as meaningful complexity reduction happens across the five most significant planes, the larger contribution of the MSB in Figs.~\ref{fig:complexity} and~\ref{fig:qubd_ptq} does not suggest that the full QuBD or its MSB estimate reduces to a sign-binarized BDM.


\subsection{Model Size Effect on Complexity Reduction Ordering}\label{app:d_effect}
Recall from Sec.~\ref{sec:saturation_effect} that the estimator fails to identify structure, and is uninformative, if the difference between $\Ccal(\x_\rho)$ and $\Ccal(\x)$ goes towards zero. In the saturated regime, retaining only a few percent of the reference gap still yields a distinguishable measure of lower- and higher-complexity models. To identify, however, if there is an effect of the symbol size $d$ exposed by the models in Fig.~\ref{fig:complexity}, we illustrate the $\Delta \Ccal_{\text{QuBD}}$ against parameter count in Fig.~\ref{fig:d_effect}. In Tab.~\ref{tab:spearman}, we further compute the Spearman correlation of the relationship between model size and $\Delta \Ccal_{\text{QuBD}}$ across the models. We find a weak relationship ($|\rho| < 0.2$) across all bit-planes $\ell=0,\ldots,7$, which suggests that informativeness is primarily attributed to the retained bit-planes, consistent with Prop.~\ref{thm:res_decomp} validated in Appx.~\ref{app:residual_validation}.

\begin{table}[htbp]
    \centering
    \caption{Spearman correlation between parameter count and $\Delta \Ccal_{\text{QuBD}}$(\%) across the 100 models in Fig.~\ref{fig:complexity}. We observe a weak relationship across all bit-lanes, suggesting the ordering of models with respect to complexity reduction does not stem from model size.}
    \label{tab:spearman}
    \begin{tabular}{lcccccccc}
        \toprule
        Plane & 7 (MSB) & 6 & 5 & 4 & 3 & 2 & 1 & 0 (LSB) \\
        \midrule
        $\rho$ & $0.188$ & $0.167$ & $0.105$ & $0.061$
                 & $0.003$ & $-0.042$ & $-0.081$ & $0.039$ \\
        \bottomrule
    \end{tabular}
\end{table}

\begin{figure}[htbp]
    \centering
    \includegraphics[width=0.9\linewidth]{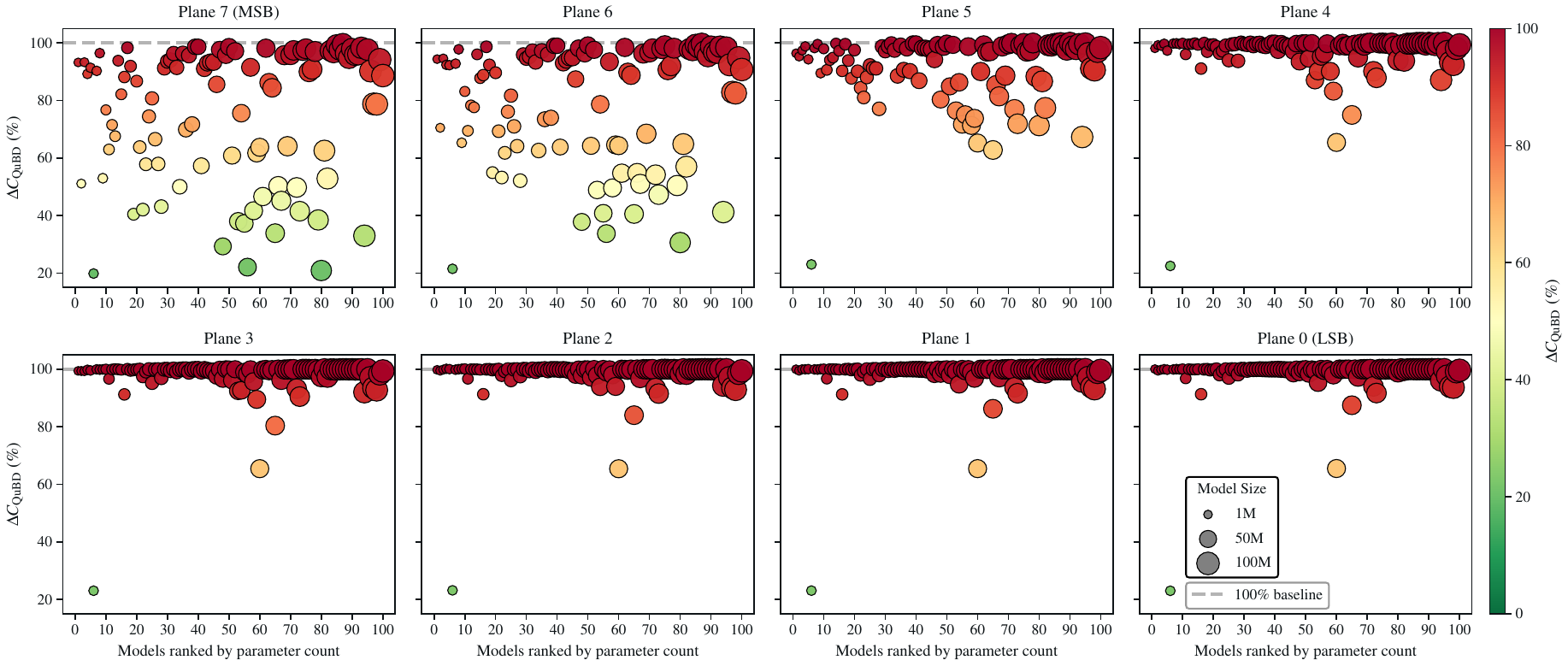}
    \caption{As Fig.~\ref{fig:complexity}. Models are ranked by their parameter count for this setting.
    }
    \label{fig:d_effect}
\end{figure}

\subsection{Quantized Bit-Planes with  GZIP}\label{app:bitplane_shannon}
As a sanity check, we replicate the results from Figs.~\ref{fig:complexity}, \ref{fig:data-budget}, and~\ref{fig:training}, using GZIP. 
GZIP estimates, when aggregated over bit-planes, however, are more discernible. Nevertheless, as estimates are not related to KCS complexity in any formal manner, unlike our QuBD complexity, we merely include them here for completeness.

\begin{figure}[hbtp]
    \centering
    \includegraphics[width=0.45\linewidth]{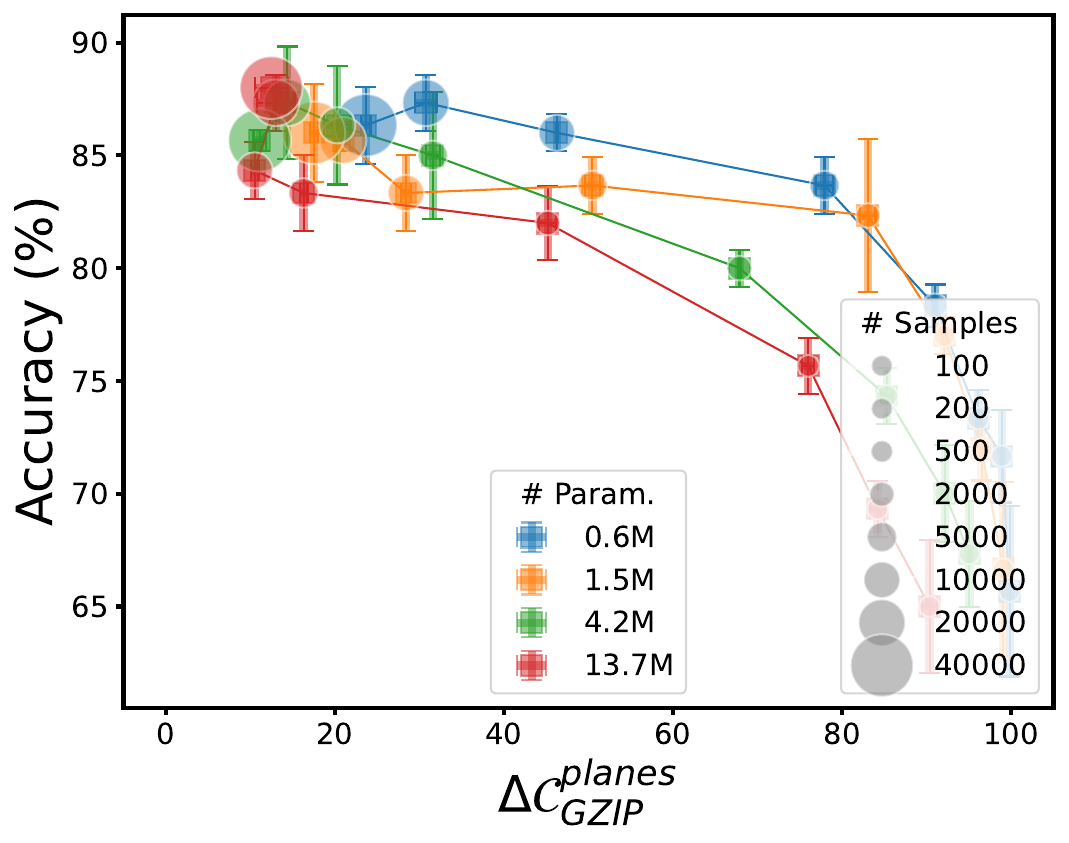}
    \caption{Relative GZIP complexity measured by summing over 8-bit planes, using the same parameters as for QuBD.}
    \label{fig:qubd_shannon_gzip}
\end{figure}

\subsection{QuBD Complexity Ratio for ResNet-50}\label{app:qubd_reduction_per_layer_resnet50}
Fig. \ref{fig:qubd_reduction_per_layer_resnet50} shows the normalized QuBD complexity per layer for ResNet-50 for the two MSB-planes (P7, P6). It shows significant differences in reduction across different layers, spanning close to $\Delta\Ccal_\text{QuBD}=0$\% and up to $\Delta\Ccal_\text{QuBD}=100$\%. Experimental details can be found in Appx.~\ref{app:per_layer}.  

\begin{figure}[htbp]
    \centering
    \includegraphics[width=0.9\linewidth]{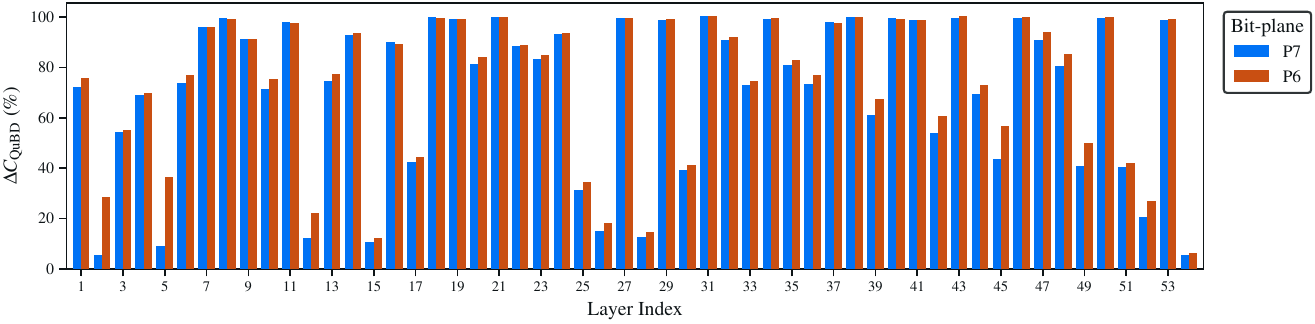}
    \caption{QuBD complexity ratio $\Delta\Ccal_\text{QuBD}$ per layer for ResNet-50 for the two most significant bit-planes (P7, P6).}
    \label{fig:qubd_reduction_per_layer_resnet50}
\end{figure}


\newpage
\clearpage
\section*{NeurIPS Paper Checklist}

\begin{enumerate}

\item {\bf Claims}
    \item[] Question: Do the main claims made in the abstract and introduction accurately reflect the paper's contributions and scope?
    \item[] Answer: \answerYes{}
    \item[] Justification: The introduction clearly states the contributions and the scope of the paper. All claims are directly supported by the theoretical results in Sec.~\ref{sec:method} and/or the experiments in Sec.~\ref{sec:method} and/or Sec.~\ref{sec:learning}.

\item {\bf Limitations}
    \item[] Question: Does the paper discuss the limitations of the work performed by the authors?
    \item[] Answer: \answerYes{}
    \item[] Justification: The limitations of the proposed QuBD method and the experiments are discussed in Sec.~\ref{sec:learning}.

\item {\bf Theory assumptions and proofs}
    \item[] Question: For each theoretical result, does the paper provide the full set of assumptions and a complete (and correct) proof?
    \item[] Answer: \answerYes{}
    \item[] Justification: The paper provides complete and correct proofs with the full set of assumptions in Appx.~\ref{proof:res_decomp} (Prop.~\ref{thm:res_decomp}), Appx.~\ref{proof:add_bitplane} (Prop.~\ref{thm:add_bitplane}), Appx.~\ref{proof:finite_saturation} (Prop.~\ref{prop:finite_saturation}) and additional details for the latter is provided in Appx.~\ref{app:saturation_details}. 

    \item {\bf Experimental result reproducibility}
    \item[] Question: Does the paper fully disclose all the information needed to reproduce the main experimental results of the paper to the extent that it affects the main claims and/or conclusions of the paper (regardless of whether the code and data are provided or not)?
    \item[] Answer: \answerYes{}
    \item[] Justification: The code used for the experiments is provided via a GitHub repository with instructions for reproducing the experiments. Experimental details are also provided in Appx.~\ref{app:experimental_details}. 

\item {\bf Open access to data and code}
    \item[] Question: Does the paper provide open access to the data and code, with sufficient instructions to faithfully reproduce the main experimental results, as described in supplemental material?
    \item[] Answer: \answerYes{}
    \item[] Justification: Code is made available via a GitHub repository with instructions for reproducing the experiments. All datasets used (Fashion-MNIST, CIFAR-10, ImageNet) are publicly available and can be downloaded without restrictions.

\item {\bf Experimental setting/details}
    \item[] Question: Does the paper specify all the training and test details (e.g., data splits, hyperparameters, how they were chosen, type of optimizer) necessary to understand the results?
    \item[] Answer: \answerYes{}
    \item[] Justification: All experimental settings necessary to understand the results are described in the main text and in Appx.~\ref{app:experimental_details}. 

\item {\bf Experiment statistical significance}
    \item[] Question: Does the paper report error bars suitably and correctly defined or other appropriate information about the statistical significance of the experiments?
    \item[] Answer: \answerYes{}
    \item[] Justification: The paper reports error bars and confidence intervals when suitable including in Tab.~\ref{tab:qat}, Fig.~\ref{fig:structure_estimate} (right), Fig.~\ref{fig:data-budget}, Fig.~\ref{fig:training}, Fig.~\ref{fig:qubd_bitplane_loss}, Fig.~\ref{fig:saturation_limits}, and Fig.~\ref{fig:rep_failure} (right), Fig~\ref{fig:rep_failure_100M} (right).

\item {\bf Experiments compute resources}
    \item[] Question: For each experiment, does the paper provide sufficient information on the computer resources (type of compute workers, memory, time of execution) needed to reproduce the experiments?
    \item[] Answer: \answerYes{} 
    \item[] Justification: This information can be found in Appx.~\ref{app:experimental_details}
    
\item {\bf Code of ethics}
    \item[] Question: Does the research conducted in the paper conform, in every respect, with the NeurIPS Code of Ethics \url{https://neurips.cc/public/EthicsGuidelines}?
    \item[] Answer: \answerYes{}
    \item[] Justification: This work conforms with the NeurIPS Code of Ethics. It involves no human subjects, raises no data-related concerns, and we foresee no negative societal impact.

\item {\bf Broader impacts}
    \item[] Question: Does the paper discuss both potential positive societal impacts and negative societal impacts of the work performed?
    \item[] Answer: \answerYes{}
    \item[] Justification: This work contributes to a better understanding of learning dynamics in deep neural networks, with direct positive societal impact through more principled model compression. By identifying redundant model capacity through algorithmic complexity analysis, QuBD can guide the development of smaller, more energy-efficient models, reducing the computational and environmental costs of deploying deep learning systems. We do not foresee direct negative societal impacts from this work, as it is foundational research focused on analysis tools rather than the development of new models or applications.
    
\item {\bf Safeguards}
    \item[] Question: Does the paper describe safeguards that have been put in place for responsible release of data or models that have a high risk for misuse (e.g., pre-trained language models, image generators, or scraped datasets)?
    \item[] Answer:  \answerNA{}.
    \item[] Justification: The paper releases only analysis code and does not introduce new models, training data, or scraped datasets. We foresee no risk of misuse associated with this work. 

\item {\bf Licenses for existing assets}
    \item[] Question: Are the creators or original owners of assets (e.g., code, data, models), used in the paper, properly credited and are the license and terms of use explicitly mentioned and properly respected?
    \item[] Answer: \answerYes{}
    \item[] Justification: All datasets, models, and code libraries used in this work are properly cited. We use publicly available datasets (Fashion-MNIST, CIFAR-10, ImageNet) and pretrained models from the \texttt{timm} library, all of which are cited accordingly. License information is not explicitly stated in the paper but all assets are used in accordance with their respective terms of use for academic research purposes.

\item {\bf New assets}
    \item[] Question: Are new assets introduced in the paper well documented and is the documentation provided alongside the assets?
    \item[] Answer: \answerYes{}
    \item[] Justification: The paper releases analysis code via a GitHub repository. The repository includes instructions for reproducing the main experiments.

\item {\bf Crowdsourcing and research with human subjects}
    \item[] Question: For crowdsourcing experiments and research with human subjects, does the paper include the full text of instructions given to participants and screenshots, if applicable, as well as details about compensation (if any)? 
    \item[] Answer: \answerNA{}.
    \item[] Justification: The paper does not involve crowdsourcing or research with human subjects.

\item {\bf Institutional review board (IRB) approvals or equivalent for research with human subjects}
    \item[] Question: Does the paper describe potential risks incurred by study participants, whether such risks were disclosed to the subjects, and whether Institutional Review Board (IRB) approvals (or an equivalent approval/review based on the requirements of your country or institution) were obtained?
    \item[] Answer:  \answerNA{}.
    \item[] Justification: The paper does not involve crowdsourcing or research with human subjects.

\item {\bf Declaration of LLM usage}
    \item[] Question: Does the paper describe the usage of LLMs if it is an important, original, or non-standard component of the core methods in this research? Note that if the LLM is used only for writing, editing, or formatting purposes and does \emph{not} impact the core methodology, scientific rigor, or originality of the research, declaration is not required.
    \item[] Answer:  \answerNA{}.
    \item[] Justification: LLMs were not used as part of the core methodology.

\end{enumerate}

\end{document}